\documentclass[11pt,a4paper,extrafontsizes,oneside]{article}
\pdfoutput=1

\usepackage[scale=0.7,vmarginratio={1:2},heightrounded]{geometry}
\usepackage{authblk} 
\usepackage{amsmath, amsfonts, amssymb}
\usepackage{amsthm}
\usepackage{array}                    
\usepackage[cmex10]{mathtools}
\usepackage[footnotesize]{caption}
\usepackage[font+=footnotesize]{subcaption}
\usepackage{url}
\usepackage{epstopdf}
\usepackage{cite}
                    
\usepackage[usenames,dvipsnames,svgnames,table]{xcolor}
\usepackage[longnamesfirst]{natbib}
\usepackage[utf8]{inputenc}
\usepackage[symbol]{footmisc}
                      
\usepackage[algo2e, ruled,vlined,linesnumbered]{algorithm2e}
\usepackage{booktabs}
\usepackage{float}

\usepackage{multirow}
\usepackage[title]{appendix}
\renewcommand{\vec}[1]{\bar{\mathbf{#1}}}

\newtheorem{theorem}{Theorem}[section]

\newtheorem{lemma}[theorem]{Lemma}

\newcommand{\vbrace}[1]{\lbrace #1 \rbrace}
\newcommand{\SMC}{SMC$^2$}
\newcommand{\iid}{\overset{i.i.d}{\sim}}
\newcommand{\etal}{et al.\ }

\newcommand{\ProposedModel}{\textsf{HQCD}}
\newcommand{\GLRT}{\textsf{GLRT}}
\newcommand{\WGLRT}{\textsf{WGLRT}}

\newcommand{\BOCPD}{\textsf{BOCPD}}
\newcommand{\RuLSIF}{\textsf{RuLSIF}}

\newcommand{\red}[1]{\color{red} #1}

\begin{document}

\title{\Large Hierarchical Quickest Change~Detection~via~Surrogates}

\author[1,2]{Prithwish Chakraborty\thanks{prithwi@vt.edu}}
\author[1,2]{Sathappan Muthiah\thanks{sathap1@vt.edu}}
\author[3]{Ravi Tandon\thanks{tandonr@email.arizona.edu}}
\author[1,2]{Naren Ramakrishnan\thanks{naren@cs.vt.edu}}
\affil[1]{Dept.\ of Computer Science, Virginia Tech, VA, USA}
\affil[2]{Discovery Analytics Center, Virginia Tech, VA, USA}
\affil[3]{Dept.\ of Electrical and Computer Engineering, University of Arizona, AZ, USA
}
\date{}

\maketitle

\begin{abstract}

Change detection (CD) in time series data is a critical problem
as it reveal changes in the 
underlying generative processes driving the time series.  Despite having received
significant attention, one important unexplored aspect is how to efficiently utilize additional correlated
information to improve the detection and the understanding of changepoints. 
We propose hierarchical quickest change detection (\ProposedModel),
a framework that formalizes the process of incorporating additional correlated 
sources for early changepoint detection.
The core ideas behind \ProposedModel~are rooted in the \emph{theory of
quickest detection} and \ProposedModel~can be regarded as its novel generalization to a hierarchical setting. The sources are classified into targets and surrogates, and  
\ProposedModel~leverages this structure to systematically assimilate
observed data to update changepoint statistics across layers.
The decision on actual changepoints are provided by minimizing the delay
while still maintaining reliability bounds.
In addition, \ProposedModel~also uncovers interesting relations between changes at targets
from changes across  surrogates. 
We validate \ProposedModel~for reliability and performance against several 
state-of-the-art methods for both synthetic dataset (known changepoints) and
several real-life examples (unknown changepoints). 
Our experiments indicate that we gain
significant robustness without loss of detection delay through 
\ProposedModel. 
Our real-life experiments also showcase the usefulness
of the hierarchical setting by connecting the surrogate sources 
(such as Twitter chatter) to target sources (such as Employment related protests
that ultimately lead to major uprisings). 
\end{abstract}

\section{Introduction}
  \label{sub:intro}

With the increasing availability of digital data sources, there is a
concomitant interest in using such sources to understand and detect 
events of interest, reliably and rapidly.
For instance, protest uprisings in unstable countries can be better analyzed
by considering a variety of sources such as economic
indicators (e.g.\ inflation, food prices)
and 
social media indicators
(e.g.\ Twitter and news activity).
Concurrently, detecting the onset of such events with minimal delay is of
critical importance. For instance, detecting a disease 
outbreak~\citep{painter2013using}
in real time can help in triggering preventive measures to control the outbreak.
Similarly, early alerts about possible protest uprisings can help in designing
traffic diversions and enhanced security to ensure peaceful protests.
\begin{figure}[h!]
\centering
\vspace{-18pt}
\includegraphics[width=0.95\columnwidth]{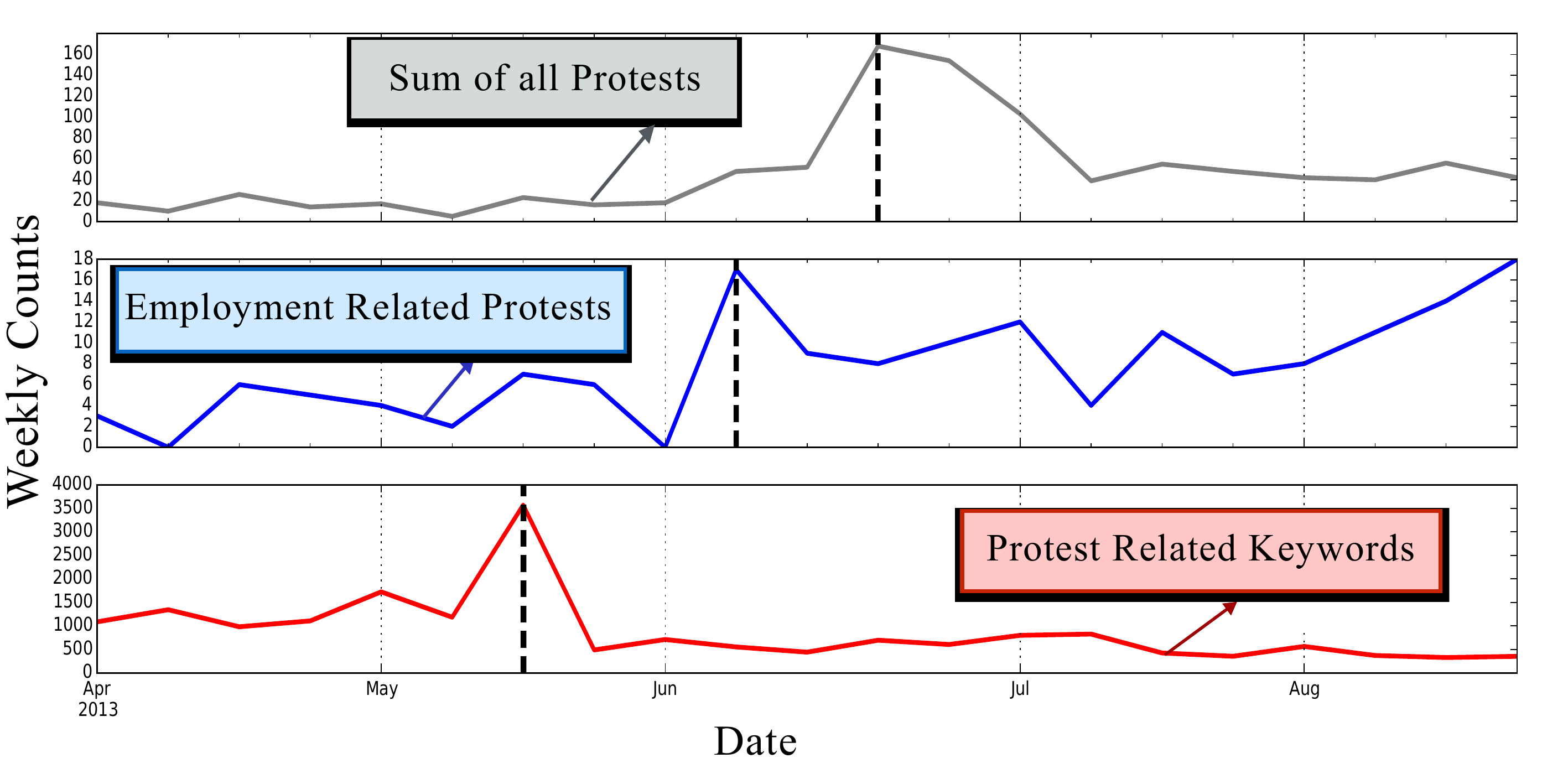}
\vspace{-10pt}
\caption{Illustrative example showing surrogate
sources which could have \textit{led to an early detection of onset} of 2013
Brazilian spring protests. (Top) Total protest counts in
Brazil over Apr'13 to Aug'13 exhibiting a sharp increase around mid-June.
(Middle) Employment and Wages related protests and (bottom) aggregated counts 
of a clusters of protest related keywords in Twitter.}
\label{fig:example}
\end{figure}

\nocite{page:cusum, 
wald1945sequential,shewhart:1925, carlin1992hierarchical,
shiryaev1963:optimum, siegmund1995using, lai1995sequential, 
lai2010sequential, adams2007bayesian, Liu201372}

\begin{table*}[tb]
\centering
\scriptsize
\caption{\label{tb:algoComparison}Comparison of state-of-the-art methods vs
Hierarchical Quickest Change Detection}
\centering
\begin{tabular}{@{}lllllll@{}}    \toprule
    {\color{blue}\textbf{Desirable}}  & Sequential                      & Window-                      & Bayesian                     & Relative             & Hierarchical & {\color{blue}\textbf{\ProposedModel}}\\
    {\color{blue}\textbf{Properties}} & GLRT                            & Limited                      & Online                       & Density-             & Bayesian     & {\color{blue}\textbf{(This}}       \\
                                      &\citeyear{shiryaev1963:optimum}  & GLRT                         & CPD                          & ratio                & Analysis of  &  {\color{blue}\textbf{Paper)}}      \\
                                      &\citeyear{siegmund1995using}     & \citeyear{lai1995sequential} & \citeyear{adams2007bayesian} & Estimation           & Change       &              \\
                                      &                                 & \citeyear{lai2010sequential} &                              & (RuLSIF)             & Point        &              \\
                                      &                                 &                              &                              & \citeyear{Liu201372} & Problems     &              \\
                                      &                                 &                              &                              &                      & \citeyear{carlin1992hierarchical} & \\
    \midrule[\heavyrulewidth]
    {{\color{red}Online}}              &\checkmark   & \checkmark & \checkmark & \checkmark &            & {\color{blue}\bf{\checkmark}}\\
    \midrule[.1pt]
    {{\color{red}Hierarchical}}        &             &            &            &            & \checkmark & {\color{blue}\bf{\checkmark}}\\
    \midrule[.1pt]
    {{\color{red}Bounded False}}\\
    {{\color{red}Alarm Rate /}}        & \checkmark  & \checkmark &            & \checkmark &            & {\color{blue}\bf{\checkmark}}\\
    {{\color{red}Detection delay}}\\
    \midrule[.1pt]
    {{\color{red}Handles}}             &             &            &            &            &            & {\color{blue}\bf{\checkmark}}\\
    {{\color{red}Non-IID data}}\\
\bottomrule
\end{tabular}
\end{table*}

Consider the evolution of the {\it Brazilian Spring}\ protests during mid June
2013 which are shown in terms of the total number of protests per week as in
Fig.~\ref{fig:example} (top panel).  This uprising can be further analyzed by
looking at individual categories of protests as shown in Fig.~\ref{fig:example}
(middle panel).  As seen, during the Brazilian Spring there was a sharp
increase in employment and wages related protests.  It is noteworthy that similar
observations can be made from Fig.~\ref{fig:example} by
observing the increase in Twitter activity for protest related keywords (bottom panel) such as ``Aborto,
Agravar, Central Dos Trabalhadores e Trabalhadoras Do Brasil'' during early
June.
This example leads to the following important observations: there is
potentially significant correlated information that can be leveraged to reduce
detection delay, and identifying the informative data source(s) can
help reduce the false positives.  Thus, appropriate usage of such
\textit{surrogate information} can potentially lead to change detection with
improved performance as well offer an interpretation behind the cause of the
changepoint.

Motivated by the aforementioned observations, we propose \textit{Hierarchical
Quickest Change~Detection} (\ProposedModel), for online change detection across multiple sources, 
viz.\ target and surrogates.  
Typically, targets are
sources of imminent interest (such as disease outbreaks or civil unrest);
whereas surrogates (such as counts of the word `protesta' in Twitter) by
themselves are not of significant interest. 
Thus, \ProposedModel~ is
aimed towards continuously utilizing both categories, but more focused on
\textit{early (or quickest) detection of  significant} changes across the
target sources. Traditional event (or change) detection approaches are not suitable for
such problems. These are either a) offline approaches~\citep{page:cusum, 
wald1945sequential,shewhart:1925, carlin1992hierarchical}
using the entire data retrospectively - thus not applicable to real-time scenarios,
or
b) online detection approaches~\citep{shiryaev1963:optimum, siegmund1995using, 
lai1995sequential, 
lai2010sequential, adams2007bayesian, Liu201372} 
with primary focus on the target source of interest and do
not utilize other correlated sources. Table~\ref{tb:algoComparison} shows a comparison of \ProposedModel~and several state-of-the-art methods in terms of the desirable attributes.

The main contributions of this paper are:\\
{\noindent $\bullet$ \ProposedModel~formalizes a {\it hierarchical structure} which in
    addition to the observed set of target sources (i.e., $S_{i}$'s),
    incorporates additional surrogates, denoted by $K_{j}$'s, and encodes
    propagation of change from surrogate to target sources.
}\\
{
\noindent
$\bullet$ \ProposedModel~presents a specialized {\it change detection metric} 
    that guarantees a maximum level of false alarm rate while reducing
    the detection delay in quickest detection framework.
    In addition, \ProposedModel~yields a natural methodology for analyzing the 
    causality of change in a particular target 
    source through a sequence of change propagations in other sources. 
}\\
{
\noindent
$\bullet$ \ProposedModel~presents a specialized sequential Monte Carlo based change
    detection framework that along with specialized change detection metrics
    enables hierarchical data to be analyzed in online fashion.
}\\
{
\noindent
$\bullet$ We {\it extensively test} \ProposedModel~on both synthetic and real
    world data.  We compare against state-of-the-art methods and illustrate the
    robustness of our methods and the usefulness of surrogates.  We also
    present a detailed analysis of three protest uprisings using real world
    data and show that the uprising could have been predicted a few weeks in
    advance by incorporating surrogate data such as Twitter chatter.
    Moreover, we analyzed target-surrogate relationships and uncover important
    propagation patterns that led to such uprisings.
}

\section{\ProposedModel--Hierarchical Quickest Change Detection}
  \label{sub:proposed}

We first provide a brief overview of classical QCD problem and then present the
HQCD framework.
\subsection{Quickest Change Detection (QCD)}
\label{sub:changepoint_overview}
Let us consider a data source $S$ changing over time and following different
stochastic processes before and after an unknown time $\Gamma$ (changepoint).
The task of QCD is to produce an estimate $\hat{\Gamma}=\gamma$ in an online
setting (i.e., at time $t$, only $S_1, \ldots, S_{t}$ is available).
Figure~\ref{fig:cpd_illustration} illustrates the two fundamental performance
metrics related to this problem. In the figure, $\Gamma = t_4$ is the actual
time-point when the changepoint happened. An early estimate such as $\gamma_1 =
t_1$ in the figure leads to a false alarm, where another estimate, such as
$\gamma_2 = t_6$ leads to an `additive delay' of $\gamma_2 - \Gamma = t_6 -
t_4$.  The goal of QCD is to design an online detection strategy which
minimizes the expected additive detection delay (EADD) while not exceeding a
maximum pre-specified probability of false alarm (PFA). QCD has been studied in
various contexts.  Some of the foremost methods have considered i.i.d.\
distributions with known (or unknown) parameters before and after unknown
changepoints~\citep{veeravalli2013quickest}.  Some of the more popular methods
have used CUSUM (cumulative sum of likelihood) based tests while  more general
approaches are adapted in GLRT (generalized likelihood ratio test) based
methods~\citep{glrt:unknownparam}.
                          
\begin{figure}[h]
  \centering
  \includegraphics[width=0.75\columnwidth]{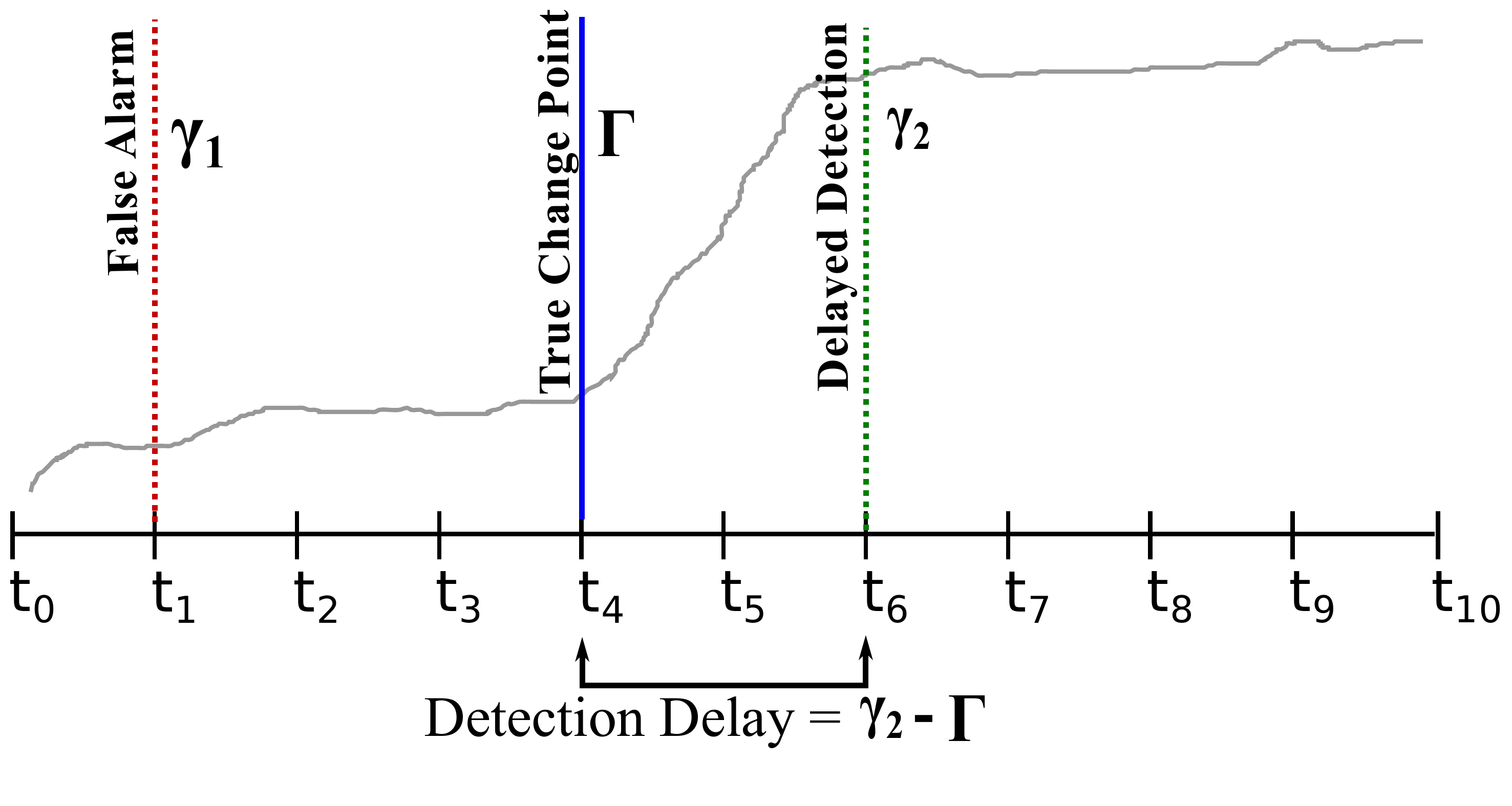}
\vspace{-12pt}
  \caption{\label{fig:cpd_illustration}
  Illustration of Quickest Change Detection (QCD): blue colored line
  represents the actual changepoint at time $\Gamma=t_4$. 
    (a) declaring a change at $\gamma_1$ leads to a
  false alarm, whereas
  (b) declaring the change at $\gamma_2$ leads to detection delay. QCD
can strike a tradeoff between false alarm and detection delay.}
\end{figure}

          \vspace{-8pt}

\vspace{0pt}
\subsection{Changepoint detection in Hierarchical Data}
\label{sub:hierarchical_cpd}
We next present our approach to generalize QCD to a hierarchical setting.  We
first describe a generic hierarchical model and then propose the QCD statistics
for such models in Section~\ref{ssub:Orignial_Formulation}.  For computational
feasibility, we present a bounded approximate of the same and our multilevel
changepoint algorithm in Section~\ref{ssub:Modified_Formulation}.

\begin{figure}[t!]
  \centering
    \includegraphics[width=0.85\columnwidth]{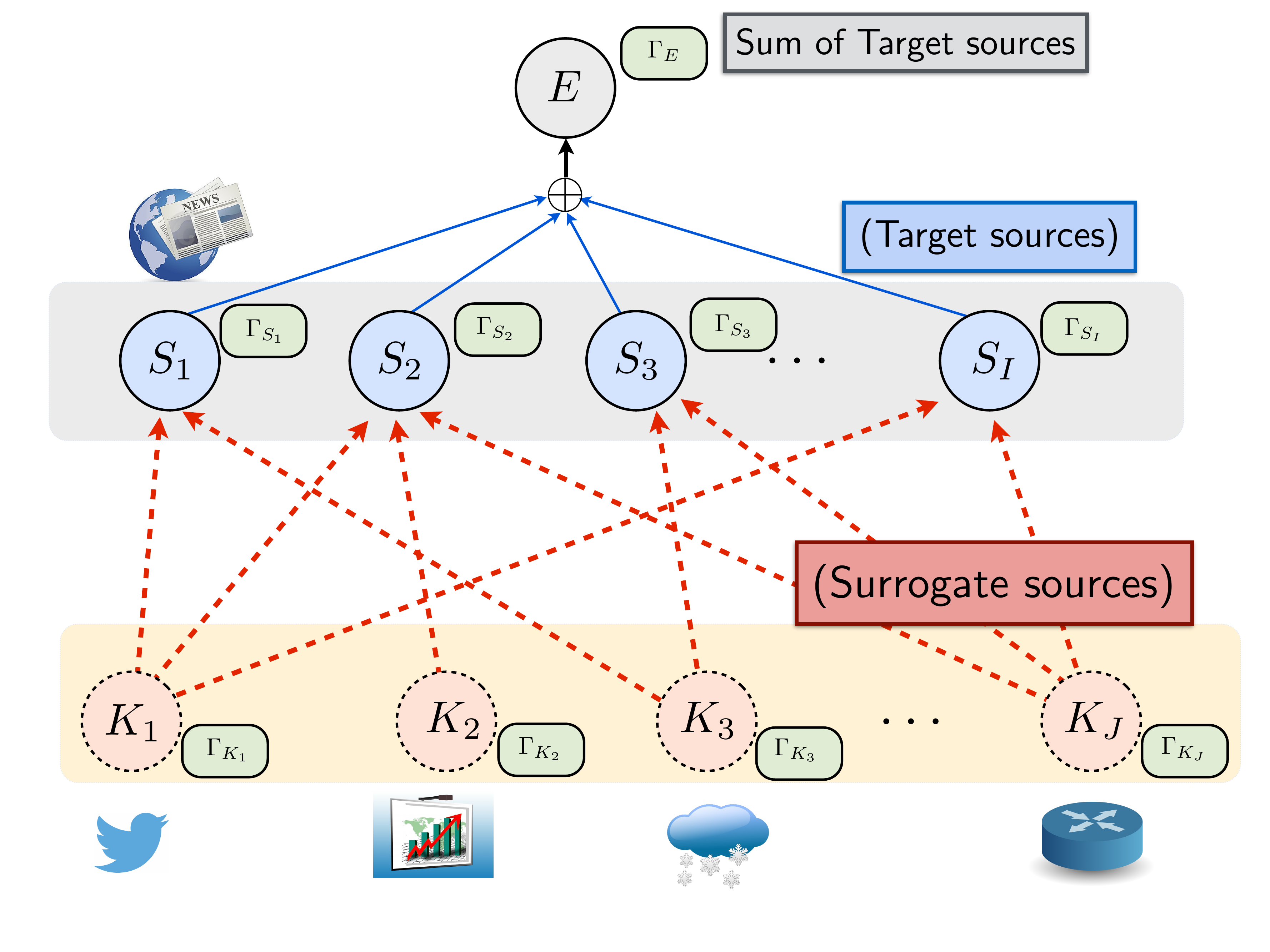}
    \vspace{-17pt}
\caption{Generative process for \ProposedModel.
As an example consider civil unrest protests. In the framework, different
protest types (such as Education- and Housing-related protests) form the
targets denoted by $S_{i}$'s. The total number of protests will be
denoted by the top-most variable $E$.  Finally, the set of surrogates,
such as counts of Twitter keywords, stock price data, weather data, network
usage data etc.  are denoted by $K_{j}$'s.
\label{fig:model}
}
\vspace{-15pt}
\end{figure}

\vspace{-4pt}
\subsubsection{Generic Hierarchical Model}
\label{ssub:generic}

Let us consider $\vec{S}^{(T)}$, a set of $I$  correlated temporal sequences
$\{S_{1}^{(T)}, S_{2}^{(T)}, \dots S_{I}^{(T)}\}$ where, $S_{i}^{(T)}$
represents the $i^{th}$ target data sequence $S_{i}^{(T)} = \left[s_{i}(1),
s_{i}(2), \dots, s_{i}(T)\right]$ for $i=1,\ldots, I$, collected up and until 
some time $T$.  The cumulative sum of the target sources $S_i$'s at time
$t$ is given by $E(t)$, i.e., $E(t) = \sum_{i=1}^{I}s_{i}(t)$. Concurrent to
target sources, we also observe a set of $J$ surrogate sources, $\vec{K}^{(T)}=
\{K_{1}^{(T)}, K_{2}^{(T)}, \dots, K_{J}^{(T)}\}$, where $K_{j}^{(T)} = \left[
k_{j}(1), k_{j}(2), \dots, k_{j}(T)\right]$, for $j=1, \ldots, J$, which may
either have a causal or effectual relationship with the target source set
$\vec{S}^{(T)}$ (see
Figure~\ref{fig:model}).
We assume that targets and surrogates follow a
stochastic Markov process as follows:

\vspace{-13pt}
\begin{align}
   P(\vec{S}^{(T)}, \vec{K}^{(T)}) = & P(S_{1}^{(T)}, \ldots, S_{I}^{(T)}, K_{1}^{(T)}, \ldots, K_{J}^{(T)}) \nonumber\\
  = & \prod_{t=1}^{T} \hspace{-2pt}\Bigg\{ \hspace{-2pt}\prod_{j=1}^{J}P_{t}^{\phi^{K}_{j}}(K_{j}(t))\times
    \prod_{i=1}^{I}P_{t}^{\phi^{S}_{i}}\hspace{-2pt}\left(S_{i}(t)| \vec{S}^{(t-1)},
    \vec{K}^{(t-1)}\right) \hspace{-3pt} \Bigg\}. \nonumber
\end{align} 
The binary variables $\phi^{K}_{j}, \phi^{S}_{i} \in \{0,1\}$ 
capture 
the notion of 
\textit{significant changes in events} through changes in distribution of the
generative process as follows: if the surrogate source $K_{j}$ undergoes a
\textit{change in distribution} at some time $t$, then, $\phi^{K}_{j}$ changes
from $0$ to $1$.  In other words, $P_{t}^{0}(K_{j})$ (respectively
$P_{t}^{1}(K_{j})$) denotes the pre-change (post-change) distribution of the
$j$th surrogate source.  Similarly, if the target source $S_{i}$ undergoes a
change in distribution at some time $t$, then $\phi^{S}_{i}$ changes from $0$
to $1$. In other words, $P_{t}^{0}(S_{i}|\cdot)$ (respectively
$P_{t}^{1}(S_{i}|\cdot)$) denotes the pre-change \break (post-change)
conditional distribution of the $j$th target data source. We denote
$\Gamma_{K_{j}}$ (respectively $\Gamma_{S_{i}}$) as the random variable
denoting the time at which $\phi^{K}_{j}$ (respectively, $\phi^{S}_{i}$)
changes from $0$ to $1$. Finally, we write $\vec{\Gamma}_{\vec{K}}=
(\Gamma_{K_1}, \ldots, \Gamma_{K_{J}})$, and $\vec{\Gamma}_{\vec{S}}=
(\Gamma_{S_1}, \ldots, \Gamma_{S_{I}})$ as the collective sets of changepoints
in the surrogate and target sources, respectively. Finally, denote
$\Gamma_{E}$ as the changepoint random variable for the top layer,
$E$, which represents the sum total of all target sources. 

\vspace{-4pt}
\subsubsection{From QCD to HQCD}
\label{ssub:Orignial_Formulation}

We extend the concepts of QCD presented in
Section~\ref{sub:changepoint_overview} to multilevel setting by formalizing the problem
as the \textit{earliest} detection of the set of all $(J+I+1)$ changepoints,
i.e., $\vec{\Gamma} = \{\vec{\Gamma}_{\vec{K}}, \vec{\Gamma}_{\vec{S}},
\Gamma_{E}\}$ having observed the target and surrogate sources i.e.
$\left(\vec{S}^{(T)}, \vec{K}^{(T)}\right)$. 
Let $\vec{\gamma}= \{\vec{\gamma}_{\vec{K}}, \vec{\gamma}_{\vec{S}},
\gamma_{E}\}$ be the $(J+I+1)$ vector of decision variables for the
changepoints.  To measure detection performance,
we define the following two novel performance criteria:
\vspace{5pt}

\noindent \textbf{\small \underline{Multi-Level Probability-of-False-Alarm (ML-PFA)}}: 
\begin{align}
\text{ML-PFA}(\vec{\gamma})= \mathbb{P}\left(\vec{\gamma} \preceq \vec{\Gamma} \right), 
\label{originalPFA}
\end{align}
where for any two $N$ length vectors $a\preceq b$, the notation implies
$a_{i}\leq b_{i}$, for $i=1, \ldots, N$. For instance, consider the example of
$I=1$ target, and $J=1$ surrogate. Then $\vec{\Gamma}= (\Gamma_{K_1},
\Gamma_{S_1})$ and $\gamma=(\gamma_{K_1}, \gamma_{S_1})$, and the probability
of multi-level false alarm is given by $\text{ML-PFA}(\gamma)=
\mathbb{P}(\gamma_{K_1}\leq \Gamma_{K_1}, \gamma_{S_{1}}\leq \Gamma_{S_{1}})$.
This definition of ML-PFA declares a false alarm only if all the $(J+I+1)$
change decision variables are smaller than the true changepoints.

\noindent \textbf{\small \underline{Expected Additive Detection Delay (EADD)}}: 
\begin{equation}
  \begin{small}
  \label{EADD}
\text{EADD}(\gamma)= \mathbb{E} \left(   | \gamma - \vec{\Gamma}|_{1} \right)
=\underbrace{\sum_{j=1}^{J} \mathbb{E}(|\gamma_{K_{j}}- \Gamma_{K_{j}}|)}_{\text{Surrogate layer delay}} 
 +\underbrace{\sum_{i=1}^{I} \mathbb{E}(|\gamma_{S_{i}}- \Gamma_{S_{i}}|)}_{\text{Target layer delay}}
 +\underbrace{\mathbb{E}|\gamma_{E}- \Gamma_{E}|}_{\text{Top layer delay}}
\end{small}
\end{equation}
 Given the observations, i.e., all target and surrogate sources $(\vec{S}^{(T)},
\vec{K}^{(T)})$ till time $T$ governed by unknown changepoints
$\vec{\Gamma}$, we aim to make an optimal decision $\gamma$ about these
changepoints under the following criterion
\begin{align}
  \label{eq:cpd_1}
 \gamma^{*}(\alpha)=  \mbox{ arg} \min \limits_{\gamma} \text{ EADD}(\gamma)
\hspace{15pt} \text{ s.t. } \text{ML-PFA}(\gamma)\leq \alpha.
\end{align}
In other words, $\gamma^{*}(\alpha)$ is the optimal change decision vector
which minimizes the EADD while guaranteeing that the ML-PFA is no more than a
tolerable threshold $\alpha$. We note that the above optimal test is
challenging to implement for real-world data sets due to following issues: a)
it requires the knowledge of pre- and post- change distributions (for all
sources) and the distribution of the changepoint random vector $\vec{\Gamma}$,
b) unlike single source QCD, finding the optimal $\gamma^{*}(\alpha)$ requires
a multi-dimensional search over multiple sources, making it computationally
expensive, and c) it does not discriminate between false alarms across
different sources. For instance, declaring false alarm at a target source (such
as premature declaration of the onset of protests or disease outbreaks) must be
penalized more in comparison to declaring false alarm at a surrogate source
(such as incorrectly declaring rise in Twitter activity). 

\vspace{-4pt}
\subsubsection{Bounded approximation of HQCD}
\label{ssub:Modified_Formulation}
We can circumvent the problem (b) of the original definition of ML-PFA as given in 
equation~\ref{originalPFA} by upper bounding it  in 
Theorem~\ref{th:modified_pfa}.
\begin{theorem}[Modified-PFA]
\label{th:modified_pfa}
Let $\vec{\gamma}=\vbrace{\vec{\gamma_S}, \vec{\gamma_K}, \gamma_E}$ be the a
set of estimates about true changepoint for targets, surrogates and
sum-of-targets, respectively.  Then under the condition of greater importance
to accurate target layer detections, ML-PFA (see~\ref{originalPFA}) is
upper-bounded by Modified-PFA, where:
\begin{equation}
\label{modPFA}
\begin{array}{l}
  \begin{small}
\text{Modified-PFA}(\gamma)\triangleq I \times \max \limits_{i}\mathbb{P}(\gamma_{S_{i}}\leq \Gamma_{S_{i}}) 
+ \min \limits_{j}\mathbb{P}(\gamma_{K_{j}}\leq \Gamma_{K_{j}}) + \mathbb{P}(\gamma_{E}\leq \Gamma_{E})
\end{small}
\end{array}
\end{equation}
\end{theorem}

\begin{proof}
We can prove the upper bound of ML-PFA with the following reductions:
\begin{equation}
\label{mod1}
\begin{array}{lcl}
\text{ML-PFA}(\gamma) &= & \mathbb{P}(\gamma \preceq \Gamma)\\
&=& \mathbb{P}(\gamma_{\vec{S}}\preceq \Gamma_{\vec{S}}, \gamma_{\vec{K}}\preceq \Gamma_{\vec{K}}, \gamma_{E}\leq \Gamma_{E})\\
&\overset{(a)}{\leq}& \mathbb{P}(\gamma_{\vec{S}}\preceq \Gamma_{\vec{S}}) + \mathbb{P}(\gamma_{\vec{K}}\preceq \Gamma_{\vec{K}}) + \mathbb{P}(\gamma_{E}\leq \Gamma_{E})\\
&\overset{(b)}{\leq}& \sum_{i=1}^{I}\mathbb{P}(\gamma_{S_{i}}\leq \Gamma_{S_{i}}) + \mathbb{P}(\gamma_{\vec{K}}\preceq \Gamma_{\vec{K}}) + \mathbb{P}(\gamma_{E}\leq \Gamma_{E})\\
&\leq& I \times \max \limits_{i}\mathbb{P}(\gamma_{S_{i}}\leq \Gamma_{S_{i}}) + \mathbb{P}(\gamma_{\vec{K}}\preceq \Gamma_{\vec{K}}) + \mathbb{P}(\gamma_{E}\leq \Gamma_{E})\\
&\overset{(c)}{\leq}& I \times \max \limits_{i}\mathbb{P}(\gamma_{S_{i}}\leq \Gamma_{S_{i}}) + \min \limits_{j}\mathbb{P}(\gamma_{K_{j}}\leq \Gamma_{K_{j}})
+ \mathbb{P}(\gamma_{E}\leq \Gamma_{E}),
\end{array}
\end{equation}
where $(a)$ and $(b)$ follows from the union bound on probability and $(c)$
follows from the fact that the joint probability of a set of events is less
than the probability of any one event, i.e.,
$\mathbb{P}(\gamma_{\vec{K}}\preceq \Gamma_{\vec{K}})\leq
\mathbb{P}(\gamma_{K_{j}}\leq \Gamma_{K_{j}})$, for any $j=1,\ldots, J$, and
then taking the minimum over all $j$. The resulting upper bound in (\ref{mod1})
leads to the basis of the modification of the multi-level PFA:
\[
\text{Modified-PFA}(\gamma) \triangleq I \times \max \limits_{i}\mathbb{P}(\gamma_{S_{i}}\leq \Gamma_{S_{i}}) \nonumber\\
 + \min \limits_{j}\mathbb{P}(\gamma_{K_{j}}\leq \Gamma_{K_{j}}) + \mathbb{P}(\gamma_{E}\leq \Gamma_{E})\label{pf:modPFA}
\]
\end{proof}
Modified-PFA expression leads to 
intuitive interpretations as follows:
 (i) as false alarms at targets can have a higher
impact, it is desirable to keep the worst case PFA across these to be the smallest, or
equivalently, $\max_i  \mathbb{P}(\gamma_{S_i} \leq
\Gamma_{S_{i}})$ should be minimized.  
 (ii) false alarms at surrogates are not as
impactful and we can declare a false alarm if all of the surrogate
level detection(s) are unreliable, or equivalently, 
$\min_j \mathbb{P}(\gamma_{K_j} \leq \Gamma_{K_{j}})$ 
needs to be minimized.  (iii) notably, the above modification leads to a
low-complexity  change detection approach across multiple sources by \textit{locally} optimal detection
strategies avoiding a multi-dimensional search.

Based on Modified-PFA, we next present a compact test suite
to declare changes at pre-specified levels of maximum PFA as given in 
Theorem~\ref{th:test_modified} and incorporate specificity issues pointed 
out in problem (c) of the original formulation of PFA.
\begin{theorem}[Multi-level Change Detection]
  \label{th:test_modified}
  Let $\Gamma_{S_i}$ be the true change point random variable for the $i$th
  target source, $S_{i}$. Let $\Gamma_{K_j}$ and $\Gamma_E$ represent the same
  for the $j$th surrogate and the sum-of-targets, respectively. Let the data
  observed till time $T$ be $D^{(T)} \triangleq \left(\vec{S}^{(T)},
  \vec{K}^{(T)}\right)$ and $P(\vec{\Gamma} | D^{(T)})$ denote the estimate of
  the conditional distribution (see Section~\ref{sub:smc_for_changepoint_detection}).
  Then, if $\alpha_i, \beta_j, \lambda$ represent the PFA thresholds for the $S_i, K_j, E$, 
  the changepoint tests can be given as:
\begin{subequations}
\label{FinalTest}
\begin{align}
  \gamma_{S_{i}}(\alpha_{i})&= \text{inf} \left\{ n:\text{TS}_{S_{i}}(D^{(T)})\geq 
      \frac{\alpha_{i}}{1+ \alpha_{i}} \right\}, i=1, \ldots, I \label{eq:FinalTest:target}\\
  \gamma_{K_{j}}(\beta_{j})&= \text{inf} \left\{ n: \text{TS}_{K_{j}}(D^{(T)})\geq
      \frac{\beta_{j}}{1+ \beta_{j}}  \right\}, j=1, \ldots, J \label{eq:FinalTest:surrogate}\\
  \gamma_{E}(\lambda)&= \text{inf} \left\{ n: \text{TS}_{E}(D^{(T)})\geq 
      \frac{\lambda}{1+ \lambda}   \right\}, \label{eq:FinalTest:top}
\end{align}
\end{subequations}
where  $\text{TS}_{X}(D^{(T)})=\mathbb{P}(\Gamma_{X}\leq n | D^{(T)})$ is the
test statistic (TS) for a source $X$. 
\end{theorem}
\begin{proof}

In quickest change detection, our goal at time $T$ is to decide if a
change should be declared for some $n\leq T$ for a particular data source.
To this end, we can use the following change detection test 
\begin{align}
\gamma_{S_{i}}(\alpha_{i})= \text{inf} \left\{  n:    \log\left(  \frac{P\left(\Gamma_{S_{i}}\leq n| D^{(T)}\right)}{P\Big(\Gamma_{S_{i}}>n| D^{(T)}\Big)} \right) \geq \log(\alpha_{i})   \right\},\nonumber
\end{align}
which is equivalent to the following test:
\begin{align}
\gamma_{S_{i}}(\alpha_{i})= \text{inf} \left\{  n:   P\left(\Gamma_{S_{i}}\leq n| D^{(T)}\right)\geq \frac{\alpha_{i}}{1+ \alpha_{i}}   \right\}.\label{Ourtest}
\end{align}
Intuitively, the above test declares the change for the $i$th target source
$S_{i}$ at the smallest time $n$ for which the test statistic (i.e., posterior
probability of the change point random variable being less than $n$) exceeds a
threshold. 
The probability of false alarm for the above test can be bounded
in terms of the threshold $\alpha_{i}$ as: 
\begin{equation}
\label{PFAbound}
\begin{array}{lcl}
\mathbb{P}(\gamma_{S_{i}} \leq \Gamma_{S_{i}})
&=& \sum_{D^{(T)}}\sum_{n}\mathbb{P}(D^{(T)}, \gamma_{S_i}=n)\mathbb{P}(\Gamma_{S_{i}}>n| D^{(T)}, \gamma_{S_{i}}=n)\\
&\overset{(d)}{\leq}& \underbrace{\sum_{D^{(T)}}\sum_{n} \mathbb{P}(D^{(T)}, \gamma_{S_i}=n)}_{=1} \left(\frac{1}{1+\alpha_{i}}\right)\\
&=& \frac{1}{1+\alpha_{i}},
\end{array}
\end{equation}

where $(d)$ follows from the fact that given the observed data and the event,
$\gamma_{S_i}=n$, i.e., the change is declared at $n$, then it follows from
(equation~\ref{Ourtest}) that 
\[\mathbb{P}(\Gamma_{S_{i}}>n| D^{(T)},\gamma_{S_{i}}=n)\leq 1/(1+\alpha_{i})\]
Let us denote the test statistic (TS) for a data source $X$ as:
\[\text{TS}_{X}(D^{(T)})=\mathbb{P}(\Gamma_{X}\leq n | D^{(T)})\]
Then, then the multi-level change detection test is: 
\begin{subequations}
\label{pf:FinalTest}
\begin{align}
  \gamma_{S_{i}}(\alpha_{i})&= \text{inf} \lbrace n:\text{TS}_{S_{i}}(D^{(T)})\geq 
      \frac{\alpha_{i}}{1+ \alpha_{i}} \rbrace, i=1, \ldots, I \nonumber\\
  \gamma_{K_{j}}(\beta_{j})&= \text{inf} \lbrace n: \text{TS}_{K_{j}}(D^{(T)})\geq
      \frac{\beta_{j}}{1+ \beta_{j}}  \rbrace, j=1, \ldots, J \nonumber\\
  \gamma_{E}(\lambda)&= \text{inf} \lbrace n: \text{TS}_{E}(D^{(T)})\geq 
      \frac{\lambda}{1+ \lambda}   \rbrace \nonumber
\end{align}
\end{subequations}
\end{proof}
From Theorem~\ref{th:test_modified},  we can infer the following boundedness
property of Modified-PFA as expressed in the following Lemma.

\begin{lemma}\label{lemmaPFA}
  If we define $\alpha \overset{\Delta}{=} \min_i (\alpha_i)$ and 
  $\beta \overset{\Delta}{=} \max_j (\beta_j)$, then Modified-PFA
  in equation~\ref{modPFA} can be bounded as:
\begin{equation}
  \text{Modified-PFA}(\gamma) \leq {I \times \frac{1}{1 + \alpha}} 
                   + {\frac{1}{1 + \beta}} +  {\frac{1}{1+\lambda}}
\end{equation}
\end{lemma}
\begin{algorithm2e}                                                                  
    \footnotesize
  \SetKwInOut{Parameter}{Parameters}
  \SetKwInOut{Input}{Input}
  \SetKwInOut{Output}{Output}                             
  \SetKw{Compute}{Compute}
  \SetKw{Find}{Find}  
  \SetKw{Update}{Update}  
  \SetKw{ReturnDecision}{Return Decision}  
  \DontPrintSemicolon                                                              
  \Input{At time $T$, Target and Surrogate Sources $D^{(T)} =\left(S^{(T)}, K^{(T)}\right)$}
      \Parameter{PFA threshold for targets ($\alpha$), surrogates ($\beta$), and sum of targets ($\lambda$) \\
       }
  \Output{Changepoint Decisions $\vec{\gamma_S}, \vec{\gamma_K}, {\gamma_E}$ 
  at each timepoint $T$}                   
  \BlankLine
  \For{each $T$ }{
    \Update joint posterior$P(\Gamma_K, \Gamma_S, \Gamma_E |D^{(T)})$\;     \tcp{target change detection}
    \For{$i \leftarrow 1$ \KwTo $I$}{
       \Compute target marginal $P(\Gamma_{S_i} | D^{(T)})$\;        \Find $\gamma_{S_i}(\alpha)$ using~\ref{eq:FinalTest:target}\; 
    } 
    $\vec{\gamma_S} \leftarrow \vbrace{\gamma_{S_1}(\alpha), \dots, \gamma_{S_I}(\alpha)}$\;
        \tcp{surrogate change detection}
    \For{$j \leftarrow 1$ \KwTo $J$}{
       \Compute surrogate marginal $P(\Gamma_{K_j} | D^{(T)})$\;        \Find $\gamma_{K_j}(\beta)$ using~\ref{eq:FinalTest:target}\; 
    } 
    $\vec{\gamma_K} \leftarrow \vbrace{\gamma_{K_1}(\beta), \dots, \gamma_{K_J}(\beta)}$\;
        \tcp{sum-of-targets change detection}
    \Compute sum-of-targets marginal $P(\Gamma_{E}| D^{(T)})$\;     \Find $\gamma_{E}(\lambda)$ using~\ref{eq:FinalTest:top}\; 
    \ReturnDecision $\vec{\gamma_S}, \vec{\gamma_K}, {\gamma_E(\lambda)}$ at $T$\;
  }
  \caption{\ProposedModel\  Multi-level Change Point Detection Algorithm
  \label{al:main}}
\end{algorithm2e} 
\vspace{-10pt}

\vspace{-5pt}
\section{HQCD for Protest Detection via Surrogates\label{sub:hierarchical_count}}

\begin{figure}[t!]
\begin{center}
  \begin{subfigure}{.49\columnwidth}
      \includegraphics[width=1.0\textwidth]{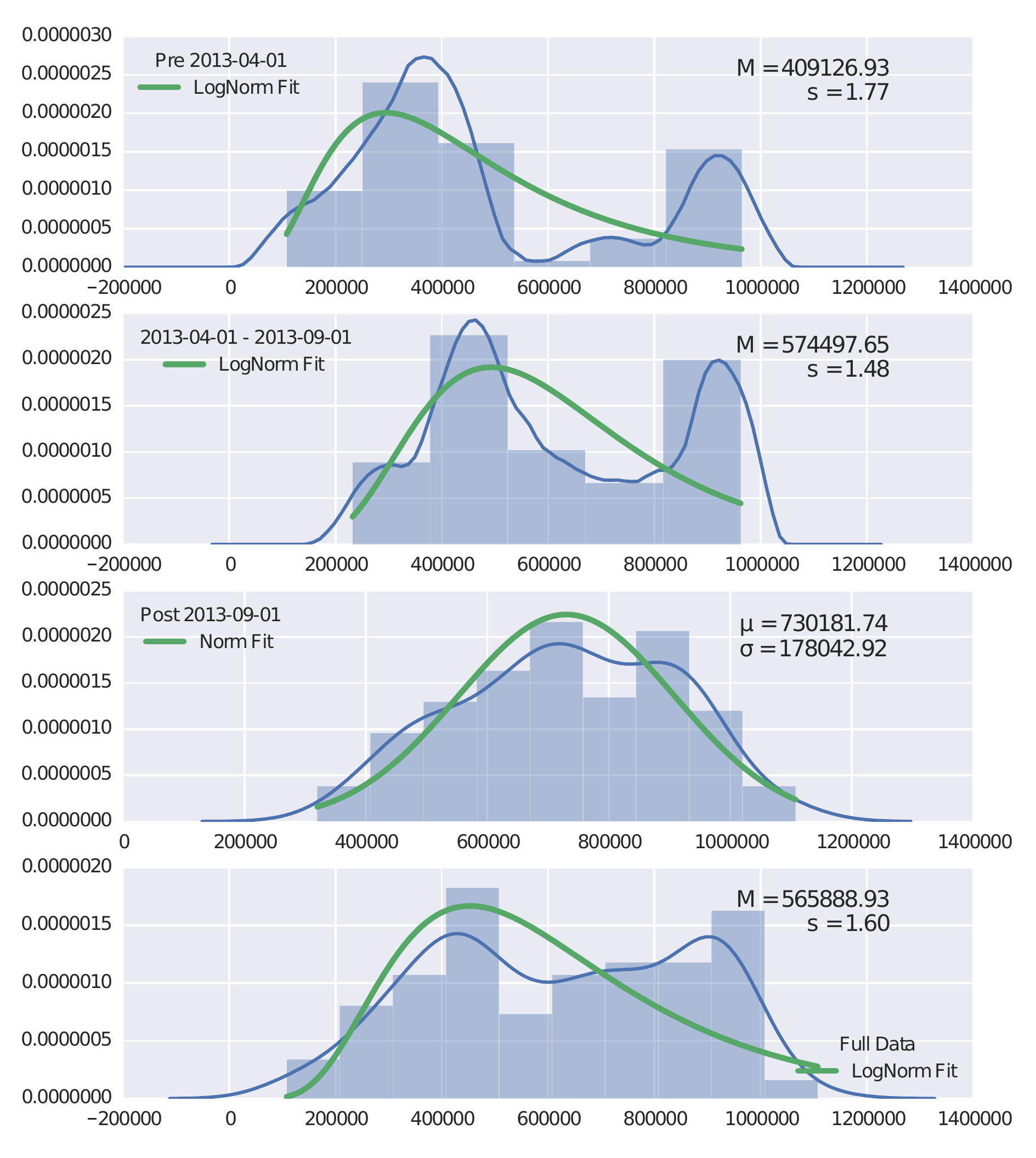}
      \caption{}
      \label{fig:keyword_fitting}
  \end{subfigure}
  \begin{subfigure}{.49\columnwidth}
          \includegraphics[width=1.0\textwidth]{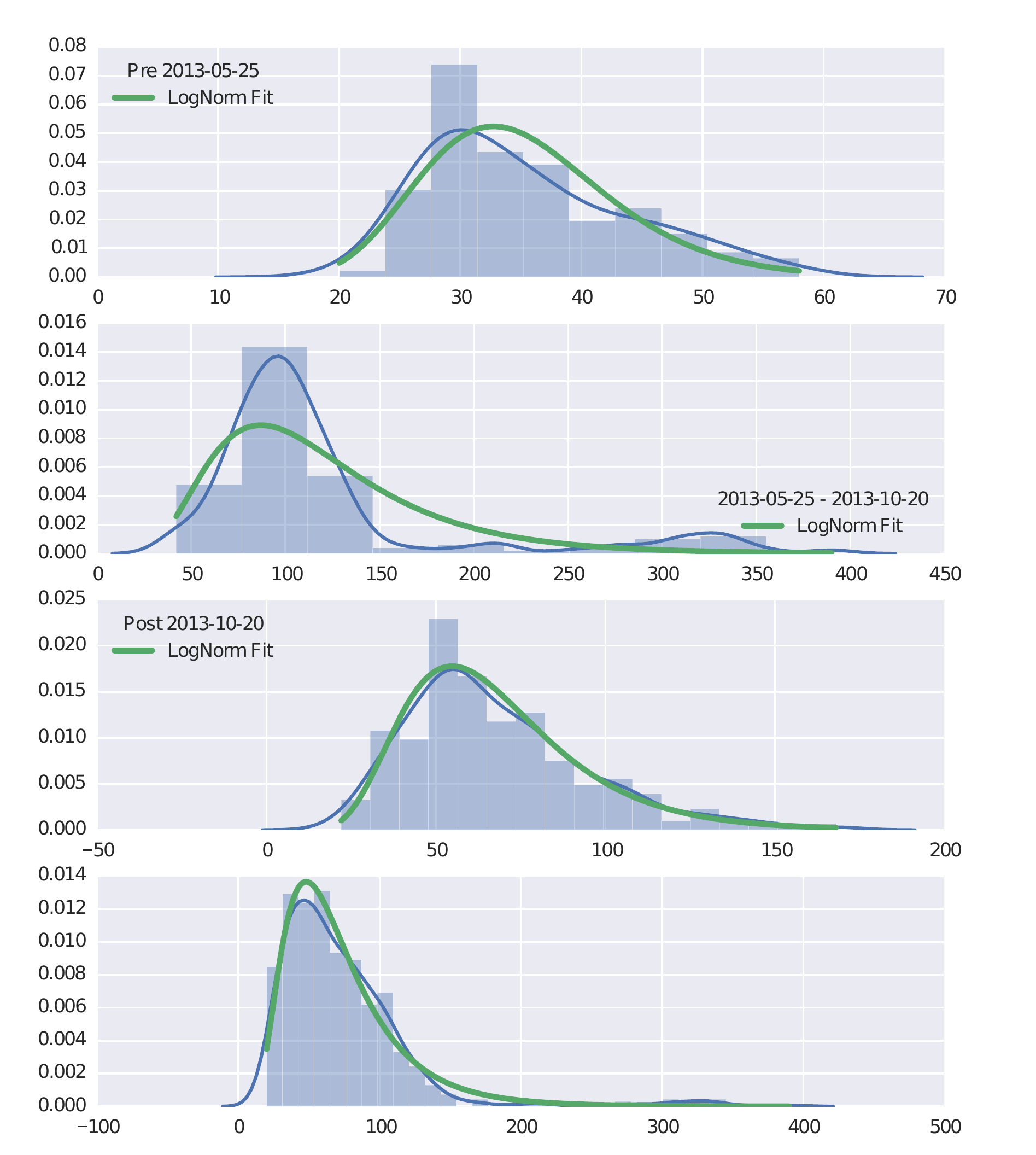}
          \caption{} 
          \label{fig:data_fitting}
  \end{subfigure}
\vspace{-10pt}
\caption{Histogram fit of (a) surrogate source (Twitter keyword counts) and 
(b) target source (Number of protests of different categories), for 
various temporal windows, under i.i.d.\ assumptions. 
These assumptions lead to satisfactory distribution fit, at a batch level,
for both sources.
The top-most row corresponds to the period before the 
Brazilian spring (pre 2013-05-25), the second row is for the period
2013-05-25 to 2013-10-20, and the third is for the period after 2013-10-20.
The last row shows the fit for the entire period.
These temporal fits are indicative of significant changes in 
distribution along the Brazilian Spring timeline, for both target
and surrogates.
\label{fig:full_fitting}
}
\end{center}
\vspace{-25pt}
\end{figure}

\vspace{-5pt}
In this section we discuss the HQCD framework for early detection of protest
uprisings via surrogate sources. Protests can happen in civil society for
various reasons such as protests against fare hike or protests demanding
more job opportunities. Such protests, especially major changes in protest 
base levels, are potentially interlinked. However explaining such interactions
is a non-trivial process. \cite{embers2014} found
several social sources, especially Twitter chatter, to capture protest related
information. We apply {\ProposedModel} to find significant changes in protests
concurrent to changes in Twitter chatter, such that detecting changes accurately
are of primary importance in contrast to the chatters which can be influenced by a range 
of factors, including protests. In general, {\ProposedModel} can be applied 
in similar events, such as disease outbreaks, to find significant changes in 
targets using information from noisy surrogates.
\vspace{-15pt}
\subsection{Hierarchical Model for Protest Count Data}
\label{ssub:Hierarchical_Model_for_Count_Data}
For the case of protest uprisings, we first note that surrogate sources such as
Twitter are in general noisy and involve a complex interplay of several
factors - one of which could be protest uprisings.  Furthermore, for protest
uprisings, we are more concerned in using the 
surrogates (Twitter chatter) to help
declare changes at target level (protest counts) 
than accurately identifying the changes in surrogates. 
Thus, without loss of generality,  we model the 
surrogates as i.i.d.\ distributed variables.
Figure~\ref{fig:full_fitting}) evaluates the i.i.d.\ assumptions, for 
both protest counts and Twitter chatter. 
Our results indicate that Log-normal is a reasonable fit for Twitter chatter. 

\noindent \textit{\underline{Surrogate Sources}}: Formally, we assume that 
the $j^{th}$ surrogate source $K_j$ is generated i.i.d.\ from a
distribution $f^{K}$ w.r.t to the associated changepoint $\Gamma_{K_j}$
as:
\begin{align}
  \label{eq:cpd_dist:surr}
k_j(t) \iid
\begin{cases}
  \begin{array}{lr}
    f^{K}(\phi^{K_j}_0) & t \leq \Gamma_{K_j} \\
    f^{K}(\phi^{K_j}_1) & t > \Gamma_{K_j}
  \end{array}
\end{cases}
\vspace{-4pt}
\end{align}
where, $\phi^{K_j}_0$  and $\phi^{K_j}_1$  are the pre- and post-change
parameters.
Following our earlier discussion, we select 
$f^{K}$ as Log-normal (with location and scale parameters
$\phi^{K_{j}} = \lbrace c^{K_{j}}, d^{K_j} \rbrace$) for Twitter counts. 

\noindent \textit{\underline{Target Sources}}: Target sources can in general be
dependent on both the past values of targets as well as the surrogates. 
Here, we restrict the target source process to be a first order Markov process.
Under this assumption, we formalize the $i^{th}$ target source $S_i$ to follow 
a Markov process $f^S_t$ w.r.t to its changepoint $\Gamma_{S_i}$ as:
\vspace{-3pt}
\begin{equation}
  \label{eq:cpd_dist:data}
s_i(t) \sim
\begin{cases}
  \begin{array}{lr}
    f^{S}_t(\phi^{S_i}_0 (t)) & t \leq \Gamma_{S_i} \\
              \vspace{1pt}
    f^{S}_t(\phi^{S_i}_1 (t)) & t > \Gamma_{S_i} \\
  \end{array}
\end{cases}
\end{equation}
where, $\phi^{S_i}_0$  and $\phi^{S_i}_1$  are the pre- and post-change
parameters of the process.
Poisson process with dynamic rate parameters has been shown~\citep{carlin1992hierarchical}
to be effective in specifying hierarchical count data w.r.t changepoints. 
Here, we model the rate parameters as a 
nested autoregressive 
process~\citep{fokianos2009poisson, carlin1992hierarchical} 
given as:
\vspace{-5pt}
\begin{equation}
  \label{eq:cpd_dist_dyn:data}
  \begin{array}{lcl}
    \phi^{S_i}_{0/1}(t) & = & \phi^{S_i}_{0/1}(t-1) + 
    \frac{A^i_{0/1}(t)}{|A^i_{0/1}(t)|} \dbinom{S(t-1)}{K(t-1)}
    + \mathcal{N}(0, \sigma_S)\\ 
    A^i_{0/1}(t) & =& A^i_{0/1}(t-1) + \mathcal{N}(0, \Sigma_{A^i})
  \end{array}
\end{equation}
Here, $\phi_{0/1}^S(t)$ captures the latent rate and $\sigma_S$ denotes the error variance.
$A_{0/1}^i(t)$ captures the variation due to the observed values of target and surrogates sources. 

\noindent  \textit{\underline{Changepoint Priors}}:
Following our prior discussion, \textit{surrogate changepoints}
can be assumed to have an uninformative prior and 
we model $\Gamma_{K_j}$ via a memoryless arrival distribution
(static probability of observing change given it hasn't occurred earlier)
as:
\vspace{-5pt}
\begin{align}
  \label{eq:cpd_prior:surr}
  \Gamma_{K_j} \sim  \text{Geom}(\rho_{K_j})  \Rightarrow  P(K_j = t | K_j \geq t) = \rho_{K_j}
\end{align}
Conversely, target changepoints can be influenced by 
surrogate changepoints as their generative process is dependent on the 
surrogates.
Specifically, whenever we observe a changepoint in the
surrogates, we assume that the base rate of changepoint for a target to increase 
for a certain period of time. Formally, \textit{target changepoint} priors are assumed to 
follow a dynamic process as:
\begin{equation}
  \label{eq:cpd_prior:data}
  \hspace{-85pt}
  \Gamma_{S_i} \sim \text{Geom}(\rho_{S_i (t)}) 
   \vspace{-5pt}
\end{equation}
\[
  \rho_{S_i}(t) = \rho_{S_i} + 
          \sum\limits_{j}\mathcal{I}(\Gamma_{K_j} < t) \mu^1_j
          e^{-\mu^2_j(t-\Gamma_{K_j})}
\vspace{-5pt}
\]
where, $\mathcal{I}$ is the indicator function. $\rho_{S_i}$ represents
the nominal base rate for the changepoint.
It can be seen, a change in the $j$th surrogate source is modeled as 
an exponentially decaying `impulse' of amplitude $\mu^1_j$.
\noindent The {\it summation of targets}, $E(t)$ is known deterministically
given $S_i(t)$. Moreover, given $S_i(t-1)$, $E(t)$ can be considered to be summation of independent 
Poisson processes following similar dynamics as  
equation~\ref{eq:cpd_dist_dyn:data} which is omitted due to limited space. 
Similarly, relationships for dependence of $\Gamma_E$ can be modeled to be dependent on $K$ similar to equation~\ref{eq:cpd_prior:data}.

\vspace{-5pt}
\subsection{Changepoint Posterior Estimation}
\label{sub:smc_for_changepoint_detection}

Algorithm~\ref{al:main} involves
posterior estimation of the changepoints given the data at a 
particular time point. 
Earlier works have focused mainly on offline methods such as 
Gibbs Sampling~\citep{carlin1992hierarchical}.
Online posterior estimation for such problems
have been studied extensively in the context of 
Sequential Bayesian Inference~\citep{casella2002statistical} such as
Kalman filters~\citep{kalman1960new, simon2010survey, anderson2001ensemble} 
(Gaussian transitions) and 
Particle Filters~\citep{del1996non, pitt1999filtering,doucet2009tutorial}.
Recently, Chopin \etal~\citep{Chopin:smc2} proposed a robust Particle 
Filter, \SMC~which is ideally suited for fitting the parameters of 
the non-linear hierarchical model described in 
Section~\ref{ssub:Hierarchical_Model_for_Count_Data}.
In this section we formulate a Sequential Bayesian Algorithm that
makes the {\ProposedModel} tractable under real world constraints
(see Figure~\ref{fig:computation}).

\begin{figure}[t!]
\centering
\includegraphics[width=0.7\columnwidth]{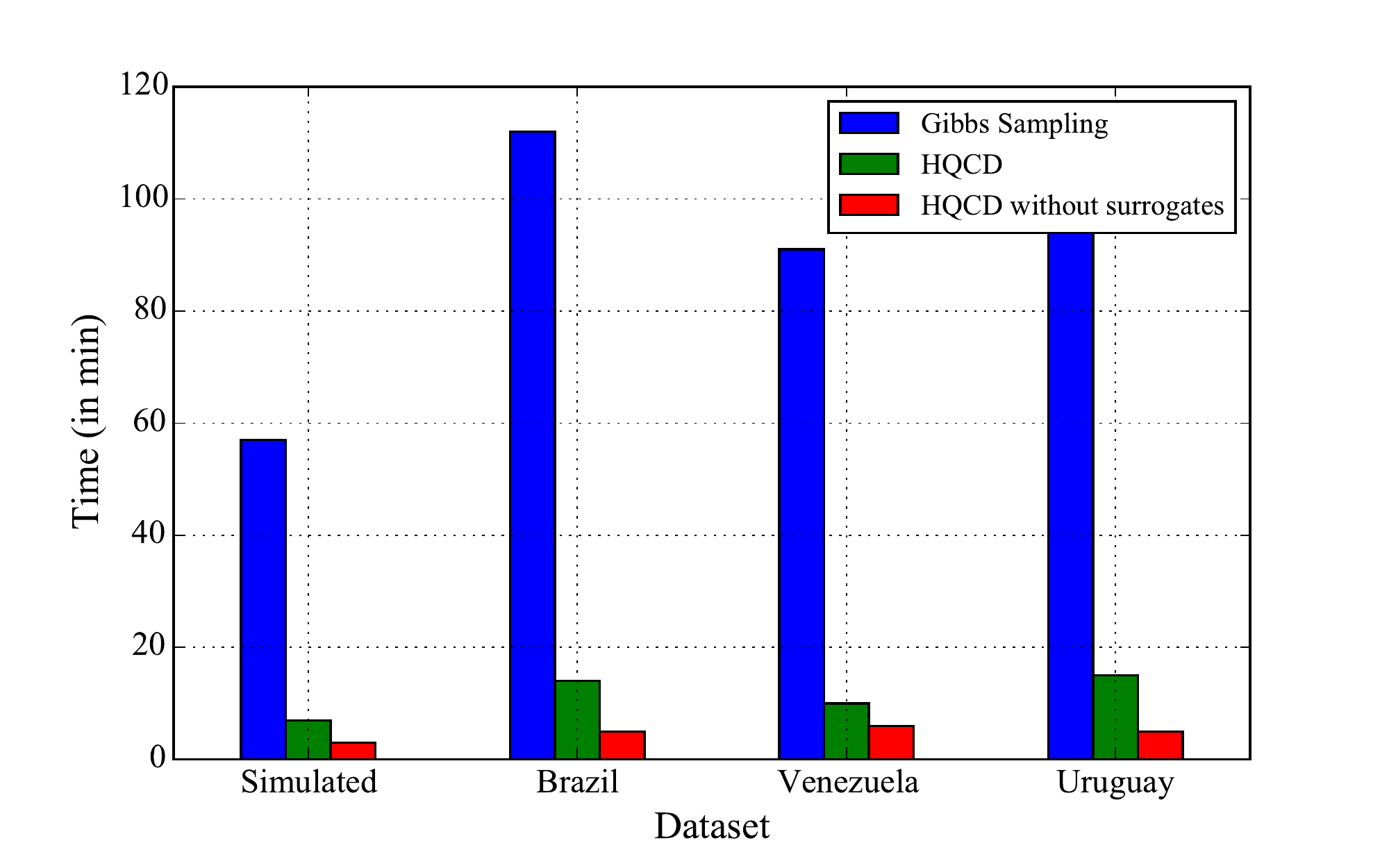}
\caption{Computation time for one complete run of changepoint detection (in mins) on
a 1.6 GHz quad core 8gb intel i5 processor:
Gibbs sampling~\citep{carlin1992hierarchical} vs \ProposedModel\ vs \ProposedModel\ without surrogates.
Gibbs sampling computation times are unsuitable for online detection.}
\label{fig:computation}
\vspace{-15pt}
\end{figure}

To find the posterior
$P\left(\vec{\Gamma}_S, \vec{\Gamma}_K, \Gamma_E | D^{(T)}\right)$ 
at any time $T$ using \SMC~we first cast the model parameters and variables 
into the following three categories:

\noindent \textit{\underline{Observations $({y_T})$}}: 
In the context of \SMC~these are the parameters that correspond to 
observed variables at each time point $T$. For \ProposedModel~we
can model $y_T$ as:
\begin{equation}
  \label{eq:smc:observation}
  {y_T} \overset{\Delta}{=} \vbrace{S(T), K(T)}
\end{equation}

\noindent \textit{\underline{Hidden States $({x_T})$}}: 
\SMC~estimates the observations based on interaction with hidden states
which are dynamic, unobserved and is sufficient to describe $y_T$ at $T$.
For \ProposedModel, we can express $x_T$ as follows:
\begin{equation}
\label{eq:smc:states}
{x_T} \overset{\Delta}{=} \{\vec{\Gamma}_S, \vec{\Gamma}_K, \Gamma_E,
\vec{\phi} _{0/1}^{S}(T-1), \vec{\phi}_{0/1}^{K},  
\vspace{-7pt}
\end{equation}
\[
\vec{\rho}_{K}(T), \vec{\rho}_{S}(T), \vec{A}_{0/1}, S(T-1), K(T-1)\}
\]

\noindent \textit{\underline{Static Parameters $({\theta})$}}: 
Finally, \SMC~also accommodates
the concept of static parameters which do not change over time such as
the base probabilities of changepoint $\vec{\rho_S}$ and 
the noise matrix $\Sigma_A$ in \ProposedModel.
We can express $\theta$ as:
\begin{equation}
  \vspace{-5pt}
  \label{eq:smc:static}
  \theta \overset{\Delta}{=} 
  \vbrace{\sigma_S, \Sigma_A, \vec{\rho_S}, \vec{\mu^1}, \vec{\mu^2}}
  \end{equation}
For a given set of such parameters, \SMC~works by first generating $N_\theta$
samples of $\theta$ using the prior distribution $P(\theta)$. 
For each of these samples of $\theta$, \SMC~samples $N_X$ samples of
$x_0$ from its prior $P(x_0|\theta)$. Following standard practices,
we use conjugate distributions~\citep{casella2002statistical}
for the priors. 

\begin{algorithm2e} 
  \footnotesize
  \SetKwInOut{Parameter}{Parameters}
  \SetKwInOut{Input}{Input}
  \SetKwInOut{Output}{Output}                             
  \SetKw{Define}{Define}
  \SetKw{Compute}{Compute}  
  \SetKw{Test}{Test}  
  \SetKw{Sample}{Sample}  
  \SetKw{Update}{Update}  
  \SetKw{ReturnUpdate}{Return Update}  
  \DontPrintSemicolon                                                              
  \Input{At time $T$, $y_T$ as give in equation~\ref{eq:smc:observation}\\
         }
      \Parameter{Prior distributions $P(\theta)$ and $P(x_0|\theta)$\\
             Hyperparameters for $P(\theta)$ and $P(x_0|\theta)$ 
             }
  \Output{joint posterior $P(\Gamma_K, \Gamma_S, \Gamma_E |D^{(T)})$}                   
  \BlankLine
  \Define $x_T$ as give in equation~\ref{eq:smc:states}\\
  \Define $\theta$ as give in equation~\ref{eq:smc:static}\\
  \BlankLine
  \tcp{Initialization}
  \Sample $N_\theta$ number of $\theta_q$ using $P(\theta)$\;
  \Sample $N_x$ number of $x_{0_{q,r}}$ using $P(x_0|\theta_q)$\;
  \Update weights $w(0)$ \tcp*{See Appendix}
  \BlankLine
  \tcp{Online Learning}
  \For{each $T$}{
    \tcp{State Updates}
    \For{each $q \in N_{\theta}$}{
       \For{each $r \in N_x$}{
       \Update States: $x_{T_{q,r}}$ from $x_{{T-1}_{q,r}}$\;
       \Compute Importance weights $w_{q,r}(T)$\;
       \Compute observation probability $P(y_T|y_{T-1}, \theta_q)$
       }
    } 
    \tcp{Incorporate observation at time $T$}
    \Update Importance weight $w_{q,r}(T) \leftarrow w_{q,r}(T)
                               P(y_T|y_{T-1}, \theta_q)$\;
    \tcp{test premature convergence}
    \Test degeneracy conditions using effective sample size\; 
    \If{degeneracy}{
       \tcp{markov kernel jumps}
       \Update $x_{T_{q,r}}$ by multiplying a markov Kernel $\mathcal{K}_T$\;
       \tcp{recomputing weights}
       exchange $x_{T_{q,r}}$ and set $w_{q_r} \propto 1$\;
    }
        \tcp{Find joints}
    \ReturnUpdate $P(\vec{\Gamma}_S, \vec{\Gamma}_K, \Gamma_E | D^{(T)})$ 
    using equation~\ref{eq:cpd_posterior}
  }
  \caption{\ProposedModel\  Changepoint Posterior estimation via \SMC
  \label{al:smc}}
\end{algorithm2e} 
At each time point  $T$, the samples are perturbed using the model equations
given in Section~\ref{ssub:Hierarchical_Model_for_Count_Data} and associated
with weights $w$ to estimate the joint posteriors as:
\begin{equation}
  \label{eq:cpd_posterior}
  \begin{array}{rcl}
  P(\theta, x_T|y_T) & = & \sum\limits_{q=1}^{N_\theta}\sum\limits_{r=_1}^{N_x} w_{q,r} \delta(\theta, x_T)\\
  P\left(\vec{\Gamma}_S, \vec{\Gamma}_K, \Gamma_E\ | D^{(T)}\right) 
   &\propto & \sum\limits_{q=1}^{N_\theta}\sum\limits_{r=_1}^{N_x} w_{q,r} 
  \delta(\vec{\Gamma}_S, \vec{\Gamma}_K, \Gamma_E)
\end{array}
\end{equation}
where, $\delta$ is the Kronecker-delta function. 
Algorithm~\ref{al:smc} outlines the steps involved in this process. For 
more details on {\SMC} see Appendix.
\vspace{-6pt}
      
\section{Experiments}
 \label{sub:experiments}

We present experimental results for both synthetic and real-world datasets,
and compare {\ProposedModel} against several state-of-the-art online change
detection methods (see Table~\ref{tb:algoComparison}), specifically,
\GLRT~\citep{siegmund1995using}, W-GLRT~\citep{lai2010sequential},
BOCPD~\citep{adams2007bayesian} and RuLSIF~\citep{Liu201372}. 
To further analyze the effects of surrogates in detecting changepoints, 
we compare against {\ProposedModel} without surrogates, 
where $K(t-1)$ is dropped 
from equation~\ref{eq:cpd_dist_dyn:data} and $\rho_{S_i(t)}$ is made static
(i.e.\ independent of changepoints from surrogates) in
equation~\ref{eq:cpd_prior:data}.

\begin{table*}[t!]
   \caption{(Synthetic data) comparing true changepoint ($\Gamma$) for targets
   against detected changepoint ($\gamma$) by \ProposedModel~against 
   state-of-the-art methods for false alarm (FA) and additive detection delay (ADD).
   Each row represent a target and best detected changepoint  is shown
   in bold whereas false alarms are shown in red.
   \label{tb:synthetic_results}}
     \centering

\begin{tabular}{@{}llllllllllllll@{}}
\toprule
& \multicolumn{1}{l}{\color{blue}{True}} & \multicolumn{2}{c}{\GLRT} & \multicolumn{2}{c}{\WGLRT} & \multicolumn{2}{c}{\BOCPD} &  \multicolumn{2}{c}{\RuLSIF} & \multicolumn{2}{c}{\color{PineGreen}{\ProposedModel}} & \multicolumn{2}{c}{\color{PineGreen}{{\ProposedModel} w/o~surr.}}\\
\cmidrule(lr){3-4}
\cmidrule(lr){5-6}
\cmidrule(lr){7-8}
\cmidrule(lr){9-10}
\cmidrule(lr){11-12}
\cmidrule(lr){13-14}

& \multicolumn{1}{l}{\color{blue}$\Gamma$} & \multicolumn{1}{c}{$\gamma$} & \multicolumn{1}{c}{\tiny{ADD}} & \multicolumn{1}{c}{$\gamma$} & \multicolumn{1}{c}{\tiny{ADD}} & \multicolumn{1}{c}{$\gamma$} & \multicolumn{1}{c}{\tiny{ADD}} & \multicolumn{1}{c}{$\gamma$} & \multicolumn{1}{c}{\tiny{ADD}} & \multicolumn{1}{c}{$\gamma$} & \multicolumn{1}{c}{\tiny{ADD}} & \multicolumn{1}{c}{$\gamma$} & \multicolumn{1}{c}{\tiny{ADD}}\\ \midrule
$S_1$      & {\color{blue}29}       & \red 7   & --  & \red 10  & --  & 13       &     & 36       & 7   & 33       & 4   & {\bf 32} & 3  \\
$S_2$      & {\color{blue}6 }       & 11       & 5   & 14       & 8   & \red 16  & 10  & 28       & 22  & {\bf 8}  & 2   & 9        & 3  \\
$S_3$      & {\color{blue}24}       & \red 7   & --  & \red 16  & --  & \red 15  &     & 29       & 5   & \red 22  & -   & 26       & 2  \\
$S_4$      & {\color{blue}26}       & \red 5   & --  & \red 11  & --  & \red 11  &     & 38       & 12  & {\bf 27} & 1   & 31       & 5  \\
$S_5$      & {\color{blue}47}       & \red 40  & --  & \red 15  & --  & \red 8   &     & \red 26  & -   & {\bf 50} & 3   & 55       & 8  \\
\bottomrule
\end{tabular}
     \end{table*}

\begin{figure*}[t]
  \centering
  \includegraphics[width=1.0\textwidth]{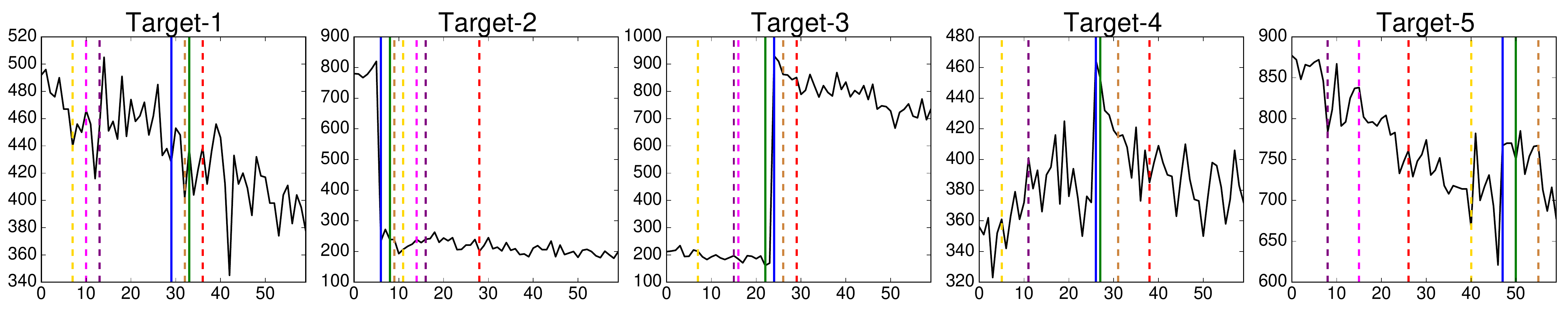}
  \caption{\label{fig:sim_data_comp}
  Comparison of \ProposedModel~against state-of-the-art on simulated target sources.
  X-axis represents time and Y-axis represents actual value.
  Solid {\color{Blue}blue} lines refer to the true changepoint, solid
  {\color{Green}green} refers to the ones detected by \ProposedModel~and
  {\color{Brown}brown} refers to \ProposedModel~without surrogates.  Dashed
  {\color{Red}red}, {\color{Magenta}magenta}, {\color{Purple}purple} 
  and {\color{Gold}gold} lines refer to changepoints detected by
  \RuLSIF, \WGLRT, \BOCPD~and \GLRT, respectively. \ProposedModel~shows
  better detection for most targets with low overall detection delay and false
  alarms. 
}
\end{figure*}

\subsection{Synthetic Data}
\label{ssub:synthetic_data}

In this section, we validate against synthetic datasets with
known changepoint parameters.
For this, we pick 5 targets ($I=5$) and 10 surrogates ($J=10$).
The surrogates were generated from i.i.d.\ Log-normal distributions (see
equation~\ref{eq:cpd_dist:surr}) while the targets were generated using Poisson
process (see equation~\ref{eq:cpd_dist:data}).  The changepoints for surrogates
were sampled from a fixed Gamma distribution (see~\ref{eq:cpd_prior:surr})
while the associated changepoints for  target sources were simulated via
equation~\ref{eq:cpd_prior:data}.

\subsubsection{Comparisons with state-of-the-art}
\label{ssub:expt:state-of-the-art}

As true changepoints are known for the synthetic dataset, we can
compare \ProposedModel~against the state-of-the-art methods for the
 detected changepoint as shown in
Figure~\ref{fig:sim_data_comp}. 
Table~\ref{tb:synthetic_results} presents the results in terms of the
false alarm (FA) and additive detection delay (ADD). From the table, we
can see that {\ProposedModel} is able to detect the changepoints with
fewer false alarms. Also \textit{{\ProposedModel} has the lowest delay} across all
methods for all targets except Target-1 for which {\ProposedModel} without
surrogates achieved better delay indicating the surrogates are
not informative for this target source.

\vspace{-2pt}
\subsubsection{Usefulness of Surrogates}
\label{ssub:Usefulness_of_Surrogates}

Our comparisons with the state-of-the-art shows significant improvements that
were achieved by {\ProposedModel}, both in terms of FA and ADD 
and \textit{showcase the importance of systematically admitting surrogate information} to 
attain a quicker change detection with low false alarm. 
We compare {\ProposedModel} with surrogates against
{\ProposedModel} without surrogates (Table~\ref{tb:synthetic_results})
and find that admitting surrogates significantly improves average delay ($2.5$ compared to $4.2$).
We also plot the average false alarm rate against the detection
delay in Figure~\ref{fig:fa_vs_delay} and find that {\ProposedModel} results
are in general the ones with the best tradeoff between FA and ADD. 

      \begin{figure}[t]
  \centering
  \includegraphics[width=0.8\columnwidth]{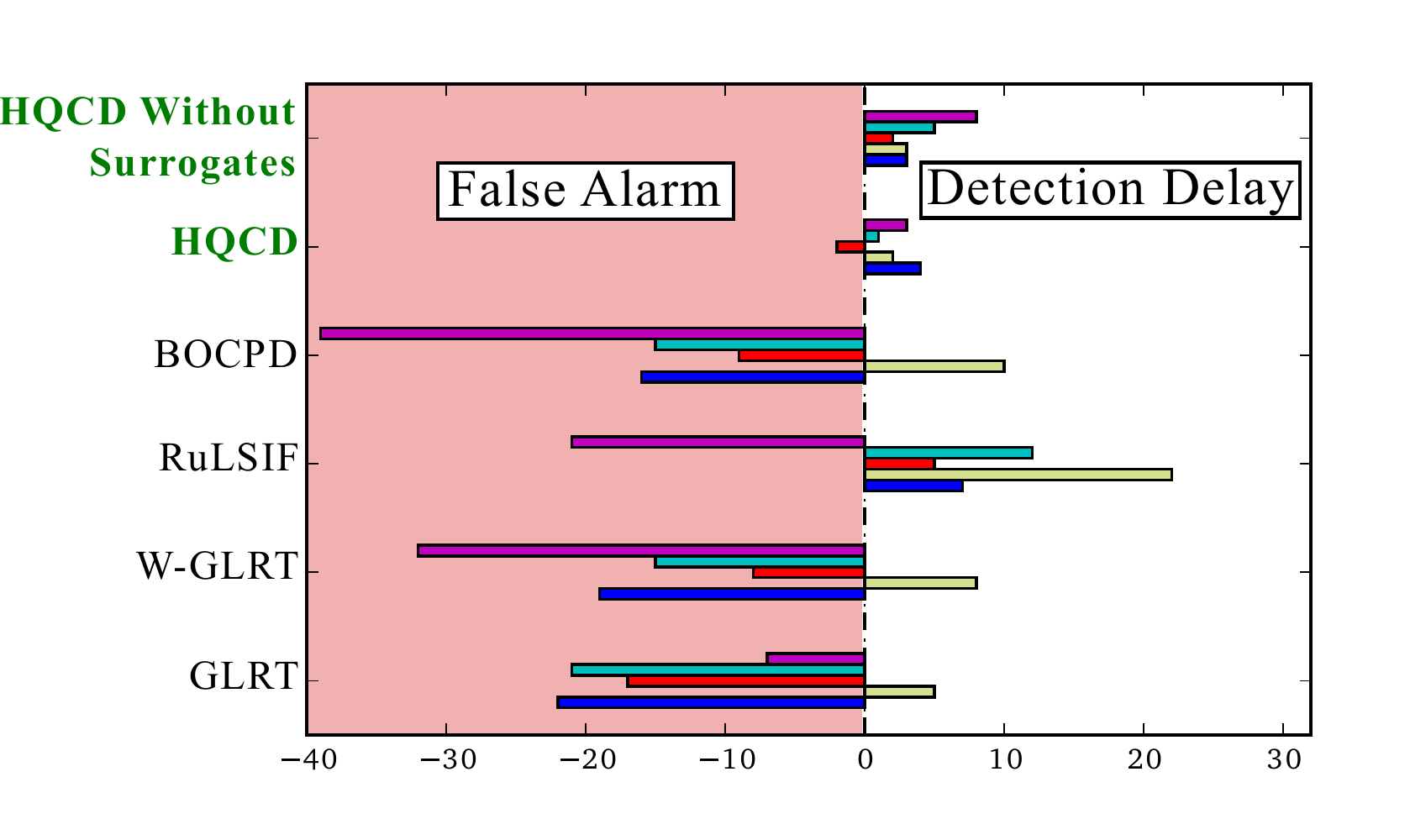}
  \caption{\label{fig:fa_vs_delay}False Alarm vs Delay trade-off for
  different methods. \ProposedModel~shows the best trade-off.}
      \end{figure}
\begin{figure}[h]
  \begin{center}
    \begin{subfigure}{.49\columnwidth}
        \includegraphics[width=1.0\textwidth]{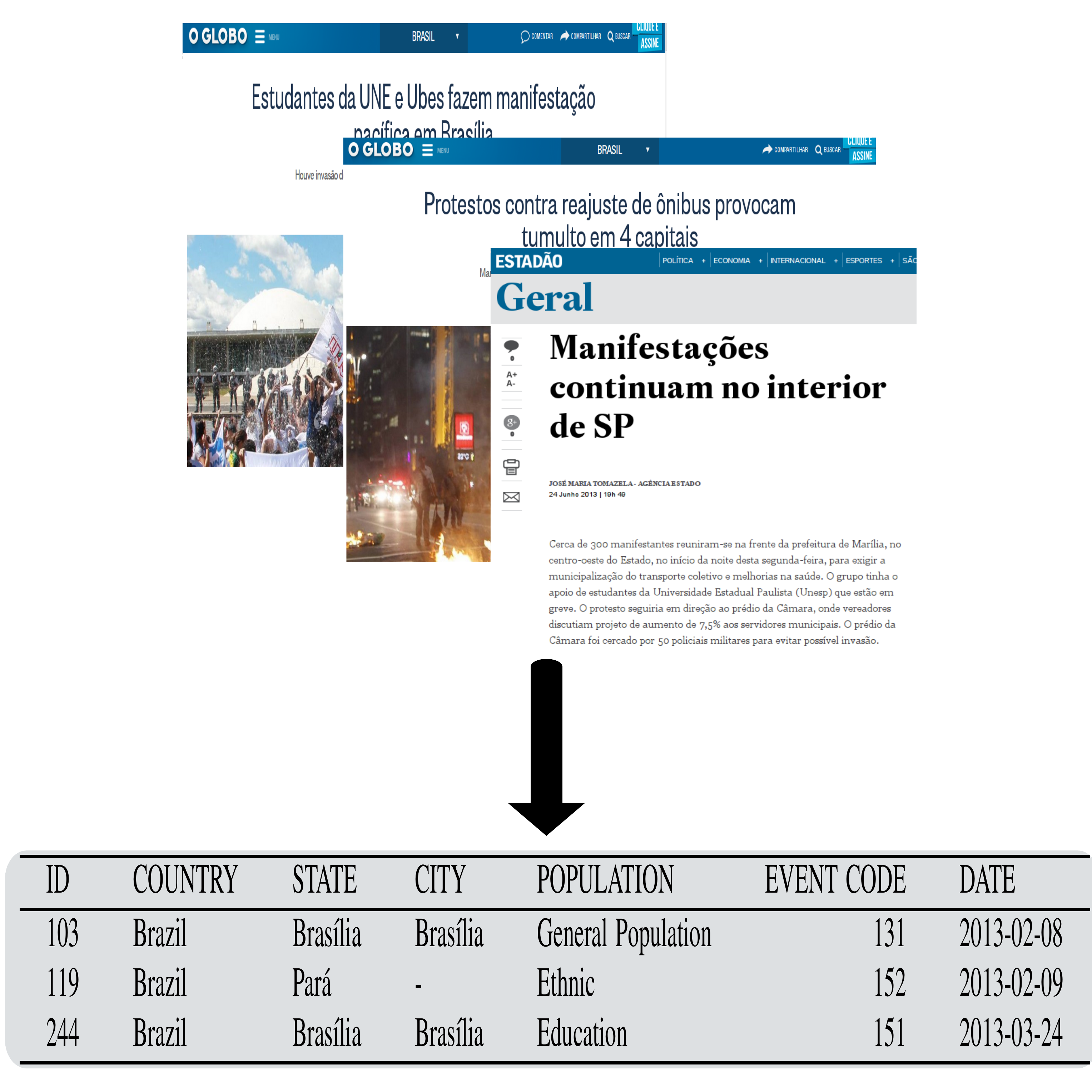}
        \caption{Civil Unrest Protests
                \label{fig:data_gsr}}
    \end{subfigure}
    \begin{subfigure}{.49\columnwidth}
            \includegraphics[width=1.0\textwidth]{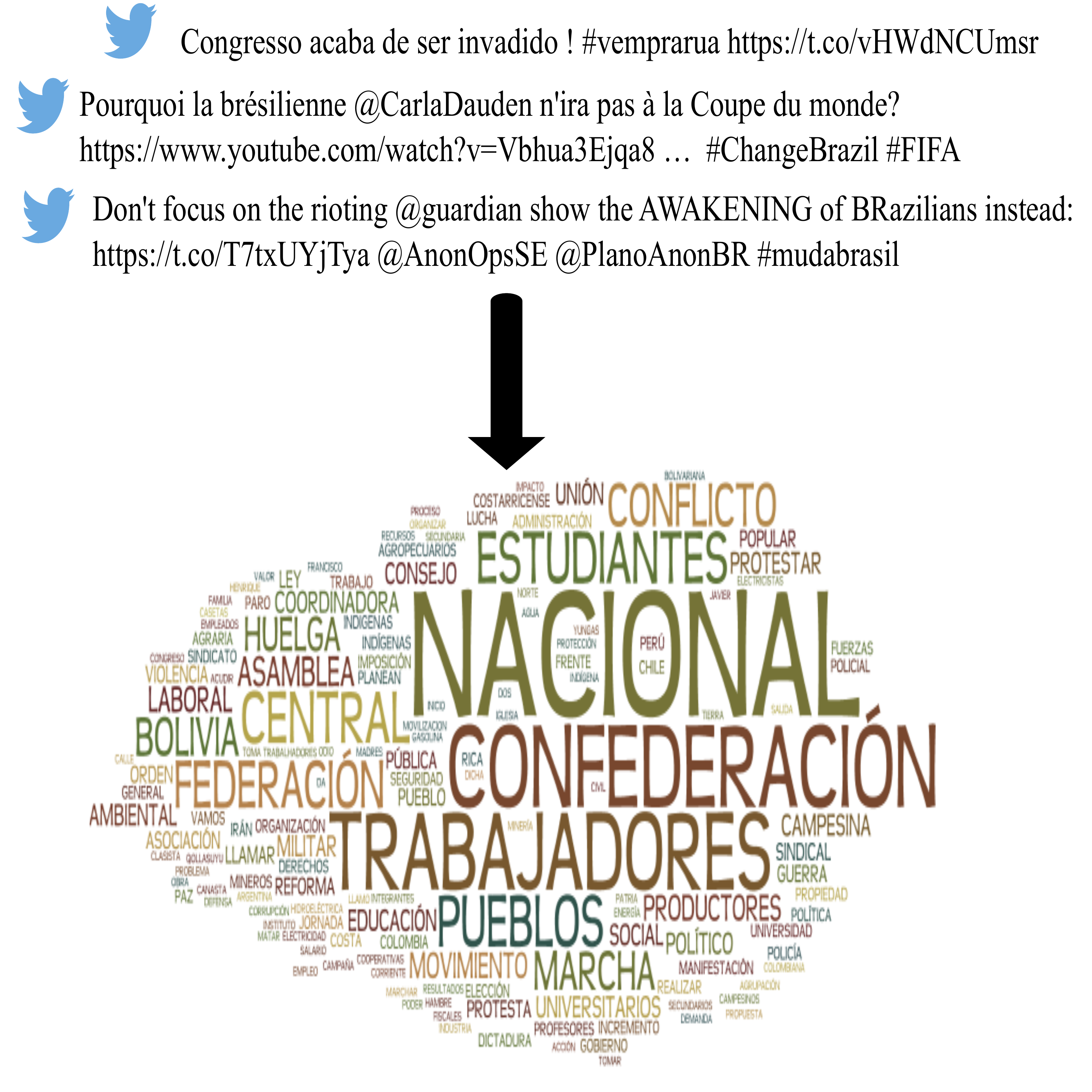}
            \caption{Twitter chatter} 
            \label{fig:data_twitter}
    \end{subfigure}
                                \caption{
    Illustration of civil unrest data: (a) shows an example where relevant news
    articles (top) are scanned to produce an annotated dataset of protest
    activities.  (b) Geo-fenced Twitter data(top). Twitter chatter can uncover
    various socio-political factors, some of which could be civil unrest
    events(bottom).\label{fig:data}
    }
\end{center}
\end{figure}

\subsection{Analysis of Protest Uprisings}
\textit{In real-life scenarios, the true changepoint is typically unknown}.
One representative example could be seen w.r.t. the onset of major 
civil unrest related protests and uprisings.  
We present a detailed analysis of three major uprisings: (i) in Brazil around
mid 2013 (often termed as the \textit{Brazilian Spring}), (ii) in Venezuela
around early 2014 and, (iii) in Uruguay around late 2013.
We first describe the data collection procedure (Figure~\ref{fig:data})
and followup with a comparative analysis of detected changepoints.

Weekly counts of civil unrest events from Nov. 2012 to Dec. 2014 were obtained as part of a database of 
discrete unrest events (Gold Standard Report - GSR) prepared
by human analysts by parsing news articles for civil unrest content.
Among other annotations, the GSR also classifies each event to one of 6
possible event types based on the reason (`why') behind the protest. Each 
of these event types such as a) \textit{Employment and Wages}, 
b) \textit{Housing}, c) \textit{Energy and Resources},
d) \textit{Other government}, e) \textit{Other economic} and f) \textit{Other}, 
bears certain societal importance. We treat the weekly counts of each of 
these event-types as target sources ($S$) and the sum total of all protests for a week
as the sum-of-targets ($E$). We also 
collected geo-fenced tweets for each country over the same time-period.
We used a human-annotated dictionary of 962 such keywords/phrases that contains
several identifiers of protest in the languages spoken in the countries of
interest (similar to Ramakrishnan et.al.~\citep{embers2014}).
As most of these keywords could have similar trends, we cluster them using
k-means into 30 clusters (i.e., we have $J=30$ surrogates). To account for 
scaling effects while preserving temporal coherence, each keyword
time-series was normalized to zero-mean and unit variance. 

\begin{table*}[t!]
\centering  
\caption{(Protest uprisings) Comparison of {\ProposedModel} vs state-of-the-art
with respect to detected changepoints}
\label{tb:cu_results}
\smaller
\begin{tabular}{@{}llllllll@{}}
\toprule
                                  &   Event-Type         & \GLRT      & \WGLRT     & \BOCPD       & \RuLSIF       & \multicolumn{2}{c}{\ProposedModel} \\
\cmidrule{7-8}
                                 &                      & $\gamma$   & $\gamma$   & $\gamma$    & $\gamma$     &  $\gamma$ & \tiny{EADD} \\
\midrule

Brazil
                                 & Employment \& Wages & 02/10      & 03/17      & 06/16       & 05/26        & 08/18     & 4    \\
                                 & Energy \& Resources  & 02/10      & 03/17      & 06/09       & 05/19        & 06/02     & 6    \\
                                 & Housing              & 03/24      & 03/31      & 07/28       & 05/19        & 06/16     & 8    \\
                                 & Other Economic       & 03/24      & 03/24      & 06/23       & 05/19        & 06/30     & 5    \\
                                 & Other Government     & 02/17      & 06/23      & 04/07       & 05/19        & 06/16     & 4    \\
                                 & Other                & 03/03      & 03/17      & 06/30       & 05/19        & 06/23     & 6    \\
                                 & All                  & 02/17      & 04/28      & 05/19       & 06/16        & 06/16     & 8    \\

\\
Venezuela                        & Employment \& Wages & 01/14      & 01/13      & 01/28       & 01/25        & 01/27     & 3    \\
                                 & Energy \& Resources  & 01/20      & 01/11      & 02/28       & 01/20        & 02/24     & 7    \\
                                 & Housing              & -          & -          & -           & -            &   -       & -    \\
                                 & Other Economic       & 01/31      & 01/31      & 01/28       & -            & 01/27     & 9    \\
                                 & Other Government     & 01/22      & 01/11      & 02/03       & 01/20        & 02/10     & 4    \\
                                 & Other                & 01/14      & 01/12      & 01/25       & 01/30        & 01/24     & 5    \\
                                 & All                  & 01/26      & 01/11      & 01/30       & 01/20        & 02/12     & 3    \\

\\
Uruguay                          & Employment \& Wages & 12/06      & 12/08      & 12/13       & 12/03        & 12/10     & 3    \\
                                 & Energy \& Resources  & 12/04      & 12/05      & 12/10       & -            & 12/09     & 4    \\
                                 & Housing              & 12/21      & 12/06      & 11/30       & -            & 11/28     & 2    \\
                                 & Other Economic       & 12/20      & 12/06      & -           & -            & 11/26     & 2    \\
                                 & Other Government     & 11/25      & 12/05      & 12/16       & 11/29        & 12/15     & 3    \\
                                 & Other                & 12/05      & 12/09      & 12/03       & -            & 01/14     & 10   \\
                                 & All                  & 12/05      & 12/09      & 12/03       & 11/29        & 12/10     & 3    \\

\bottomrule
\end{tabular}
\end{table*}

\begin{figure*}[!t]
 \begin{center}
                                                                                      
          \begin{subfigure}{0.32\textwidth}
       \centering
         \includegraphics[width=0.9\textwidth]{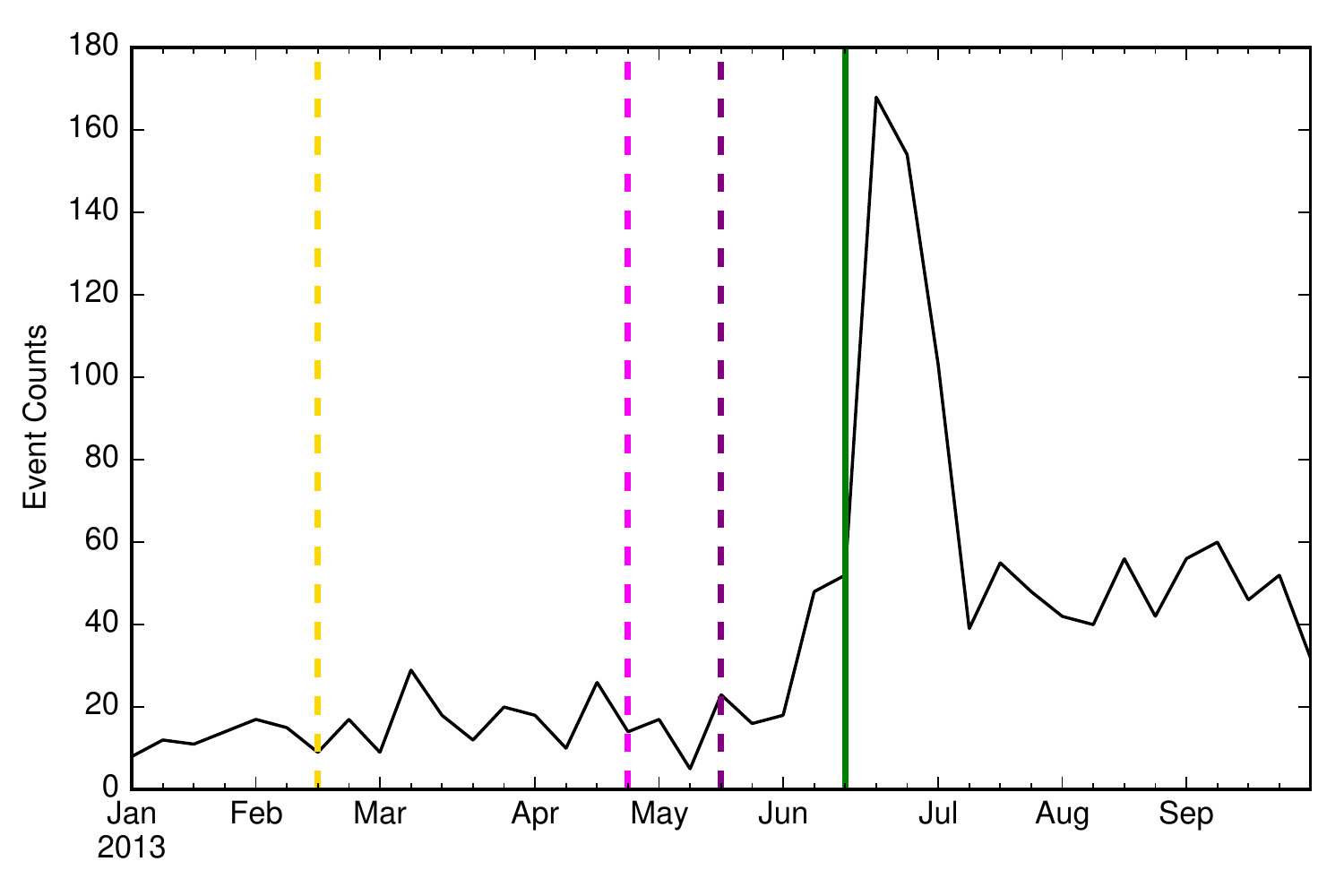}
         \vspace{-.75em}
        \caption{Brazil Total Protests\label{fig:cu_comparison:br_full}}
     \end{subfigure}     \begin{subfigure}{0.32\textwidth}
       \centering
         \includegraphics[width=0.9\textwidth]{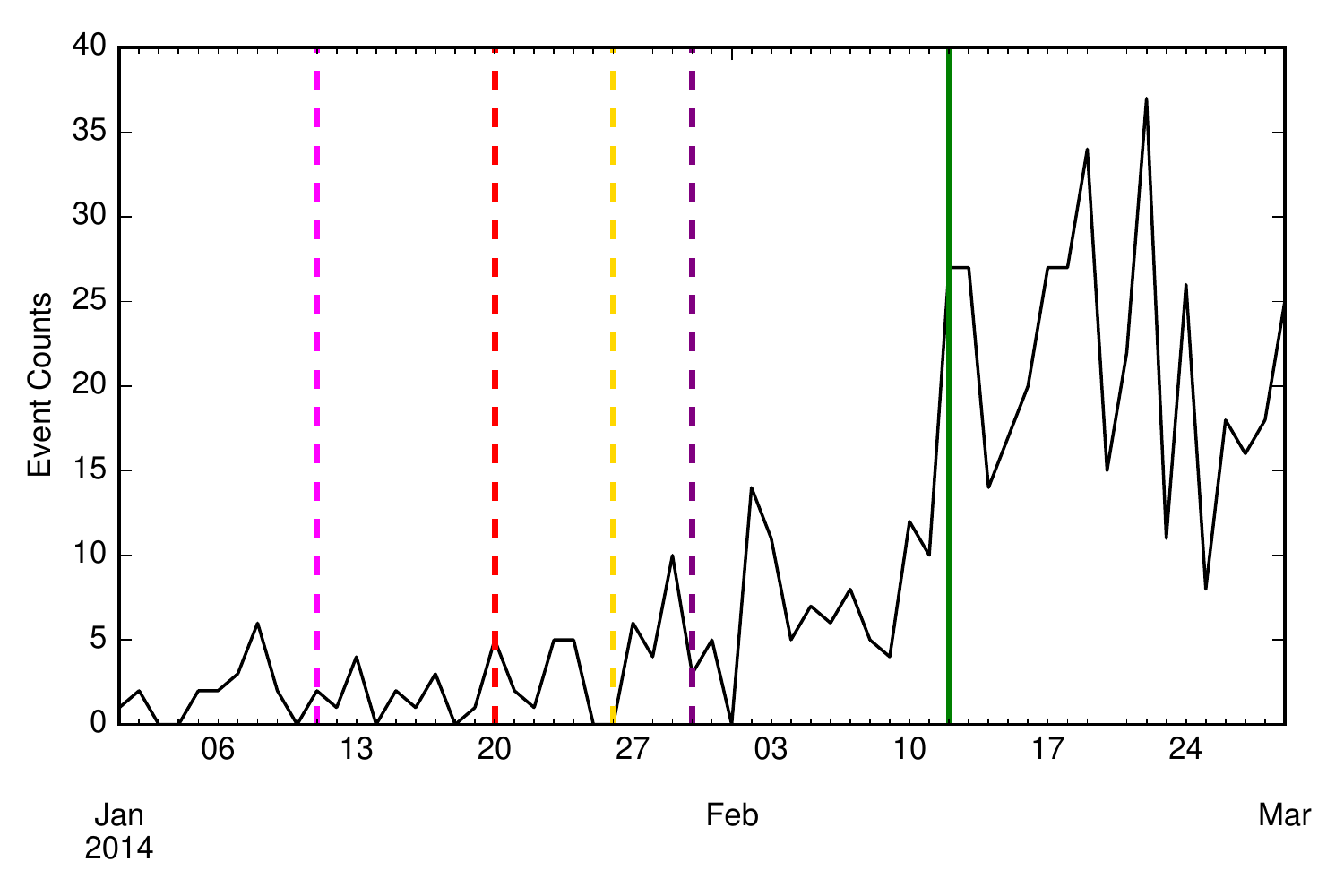}
         \vspace{-.75em}
        \caption{Venezuela Total Protests\label{fig:cu_comparison:ve_full}}
     \end{subfigure}
     \begin{subfigure}{0.32\textwidth}
       \centering
        \includegraphics[width=0.9\textwidth]{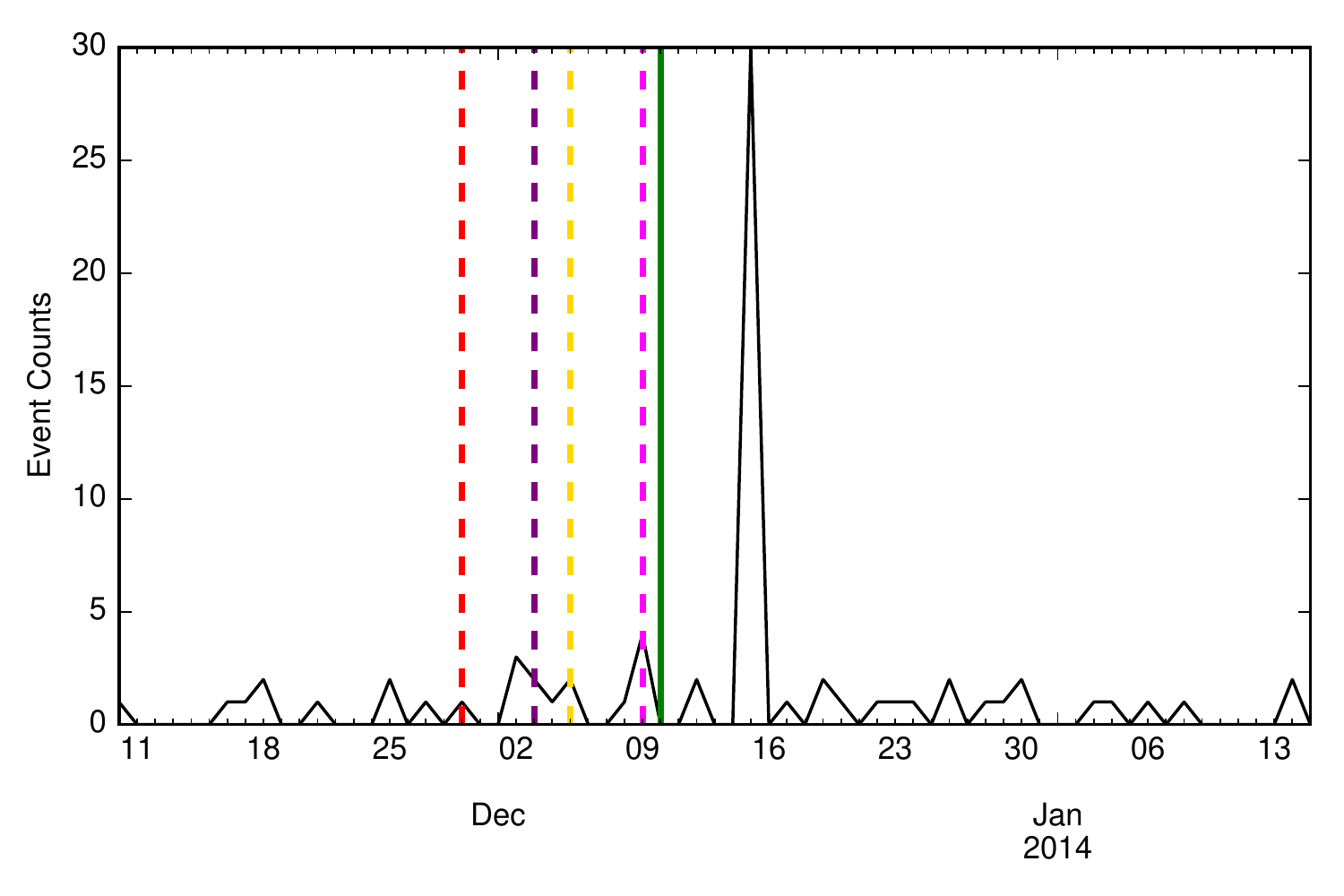}
         \vspace{-.75em}
        \caption{Uruguay Total Protests\label{fig:cu_comparison:ur_full}}
     \end{subfigure}

  \end{center}
\caption{Comparison of detected changepoints at the sum-of-targets 
(all Protests). \ProposedModel~detections are shown in
solid {\color{green}green} while those from the state-of-the-art methods i.e. 
\RuLSIF~({\color{red}red}), \WGLRT~({\color{magenta}magenta}), 
\BOCPD~({\color{Purple}purple}) and \GLRT~({\color{Gold}gold}) are shown with 
dashed lines. \ProposedModel~detection is the closest to the traditional
start date of Mass Protests in the three countries studied .
\label{fig:cu_meta}}
\end{figure*}

\begin{figure*}[t!]
  \begin{center}
    \begin{subfigure}{0.31\textwidth}
      \centering
      \includegraphics[width=0.9\textwidth]{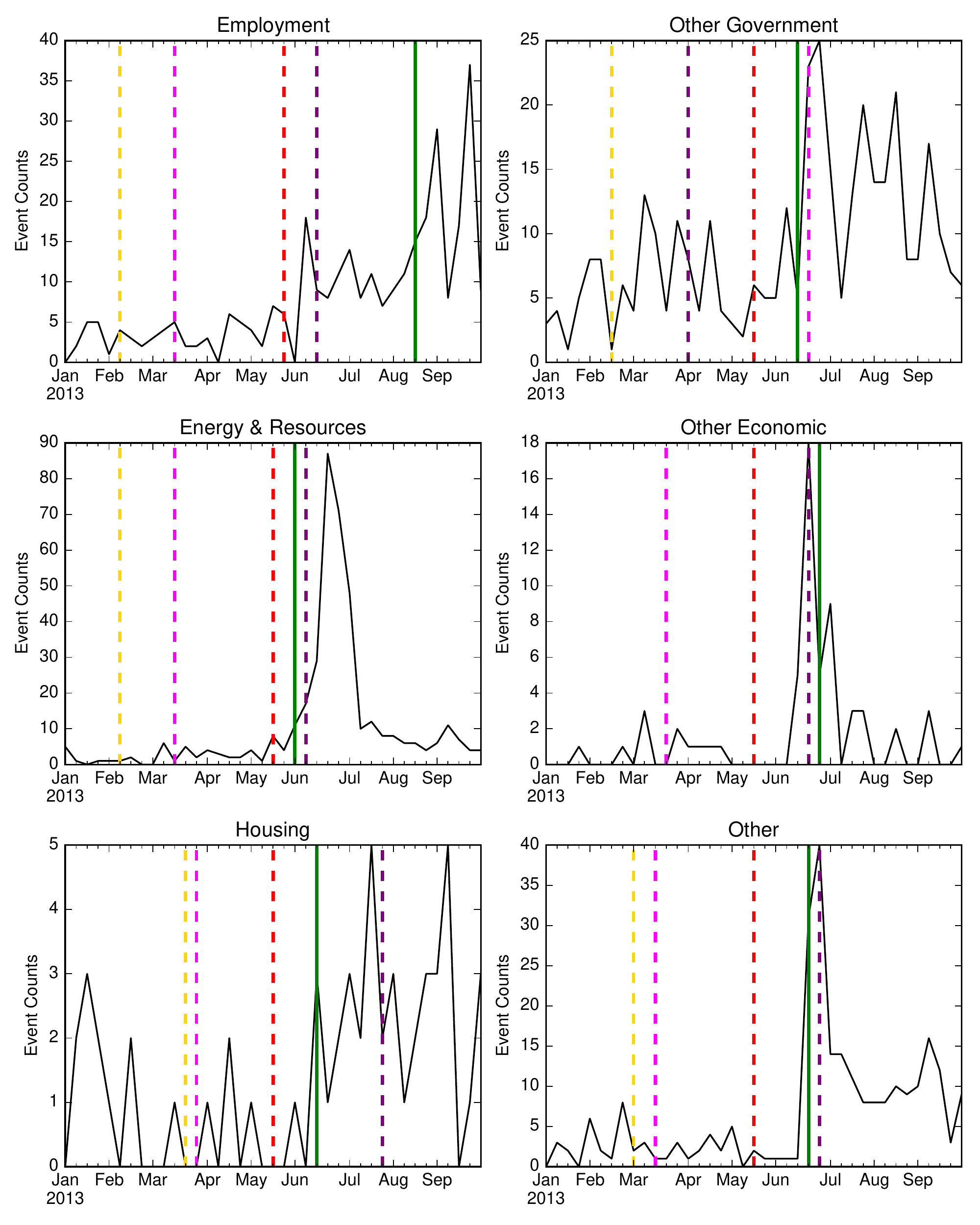}
      \caption{Brazil Subtypes\label{fig:cu_comparison:br_sub}}
    \end{subfigure}    \begin{subfigure}{0.31\textwidth}
      \centering
      \includegraphics[width=0.9\textwidth]{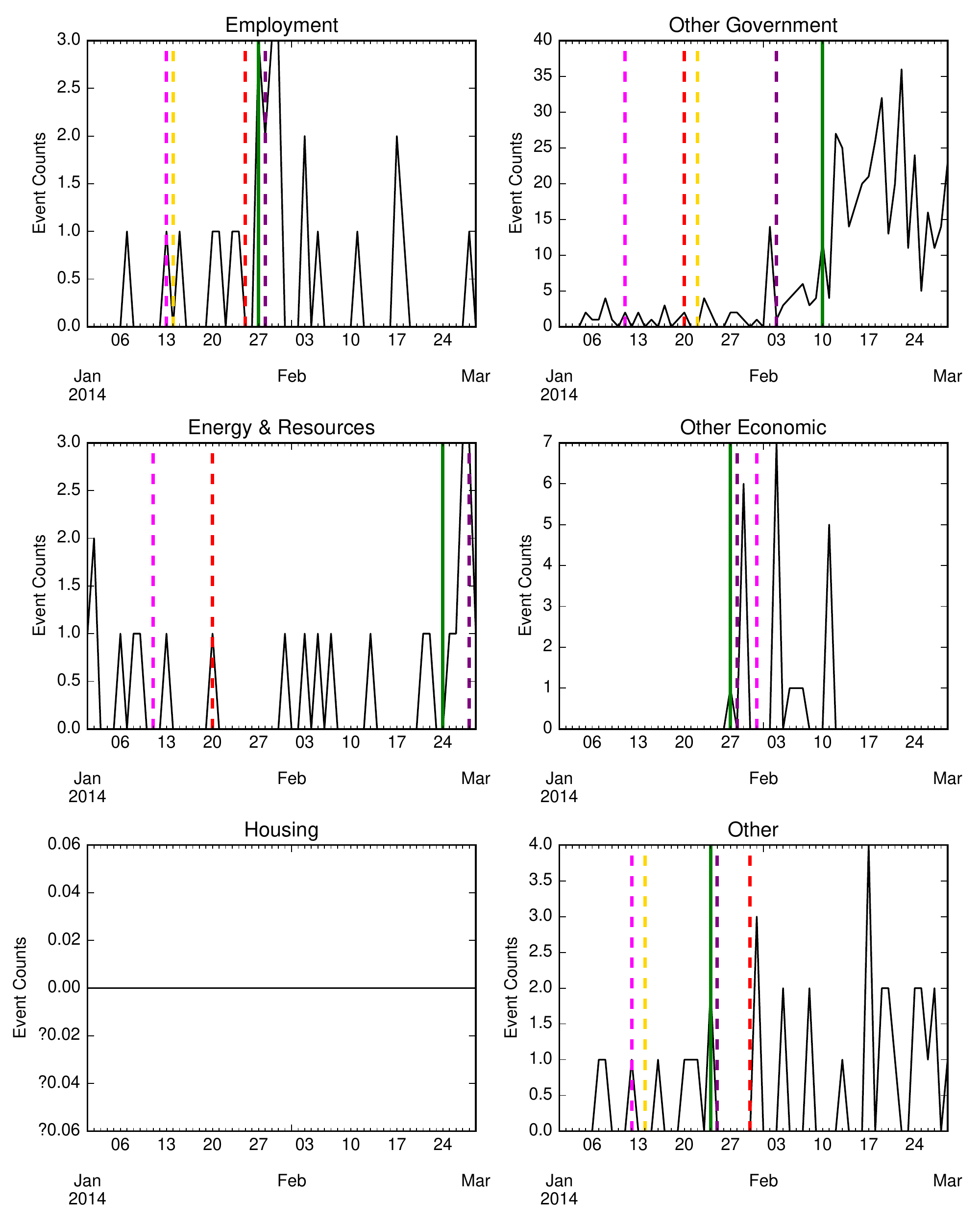}
            \caption{Venezuela Subtypes\label{fig:cu_comparison:ven_sub}}
    \end{subfigure}    \begin{subfigure}{0.31\textwidth}
      \centering
      \includegraphics[width=0.9\textwidth]{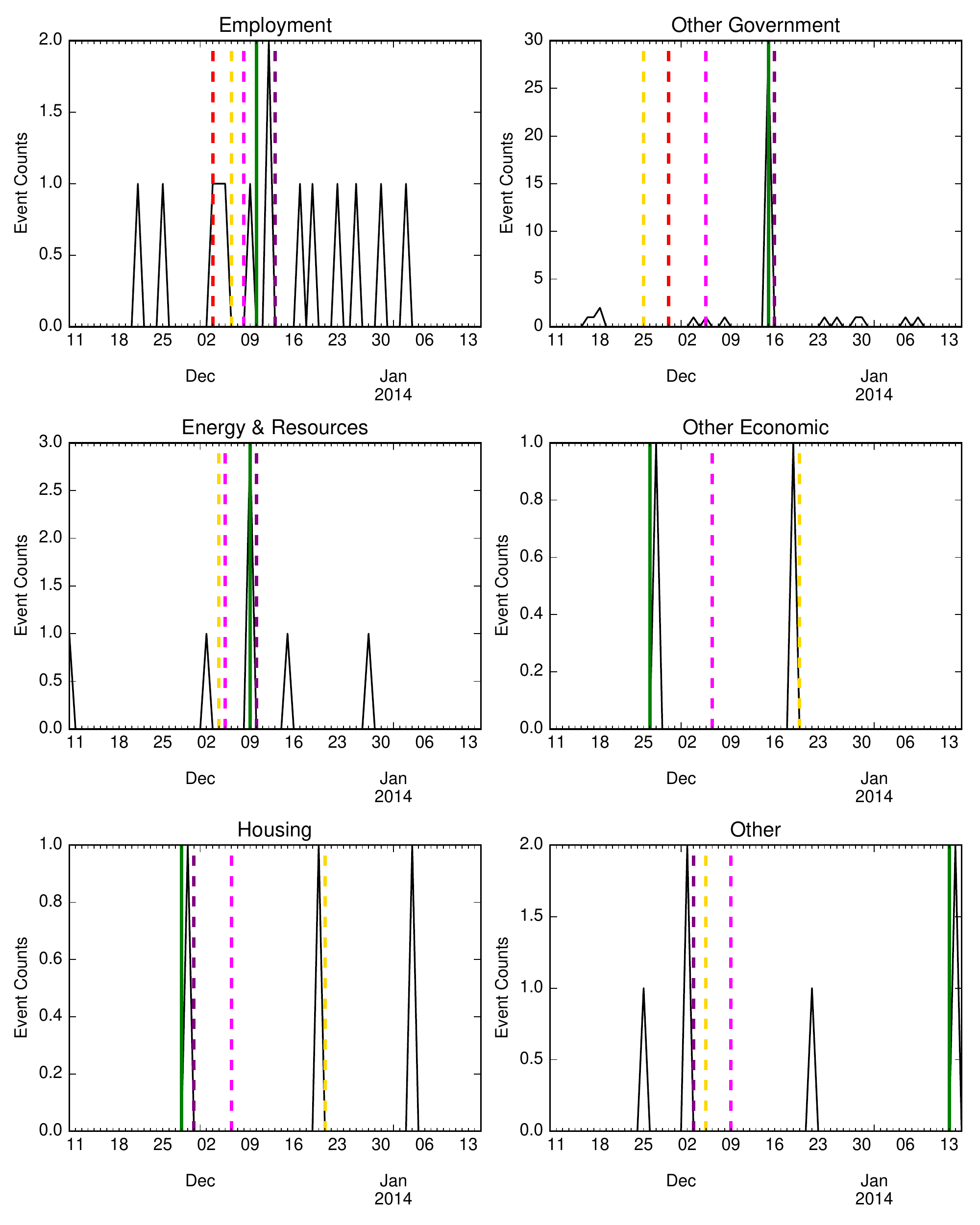}
      \caption{Uruguay Subtypes\label{fig:cu_comparison:ur_sub}}
    \end{subfigure}  \end{center}
\vspace{-15pt}
\caption{Comparison of detected changepoints at the target sources (Protest types)
\ProposedModel~detections are shown in
solid {\color{green}green} while those from the state-of-the-art methods i.e. 
\RuLSIF~({\color{red}red}), \WGLRT~({\color{magenta}magenta}), 
\BOCPD~({\color{Purple}purple}) and \GLRT~({\color{Gold}gold}) are shown with 
dashed lines. 
\label{fig:cu_meta:ap}}
\vspace{-1em}
\end{figure*}

\subsubsection{Changepoint Across layers}
\label{ssub:cu_changepoints}
We show the changepoints detected by
\ProposedModel~(bold green) and the state-of-the-art methods (dashed lines) 
for the sum-of-all protests in Figure~\ref{fig:cu_meta} and 
individual protest types in Figure~\ref{fig:cu_meta:ap}.
We can observe that \ProposedModel, which uses the
surrogate information sources and exploits the hierarchical structure, finds
indicators of changes which are visually better as well as 
more aligned to the dates of major events (See 
demo at \url{https://prithwi.github.io/hqcd_supplementary}).
In contrast, the state-of-the-art
methods can be argued to show significantly high false alarm rate. 
For such real world data sources, the notion of a \textit{true} changepoint is
difficult to ascertain, we can instead consider for example the onset of
\textit{Brazilian spring} protests (2013-06-01) as an underlying changepoint to
compare at the sum-of-targets and interpret notions of false alarm.
Table~\ref{tb:cu_results} tabulates these inferences for the targets as well as
the sum-of-targets. Although, a true changepoint is unknown, we note that for
\ProposedModel,~the expected additive detection delay (EADD) can be estimated
according to equation~\ref{EADD} (from $P(\vec{\Gamma}|D^{(T)})$ in
Algorithm~\ref{al:smc}).

                                  \begin{figure}
    \begin{subfigure}{0.9\columnwidth}
  \centering
    \includegraphics[width=0.9\columnwidth]{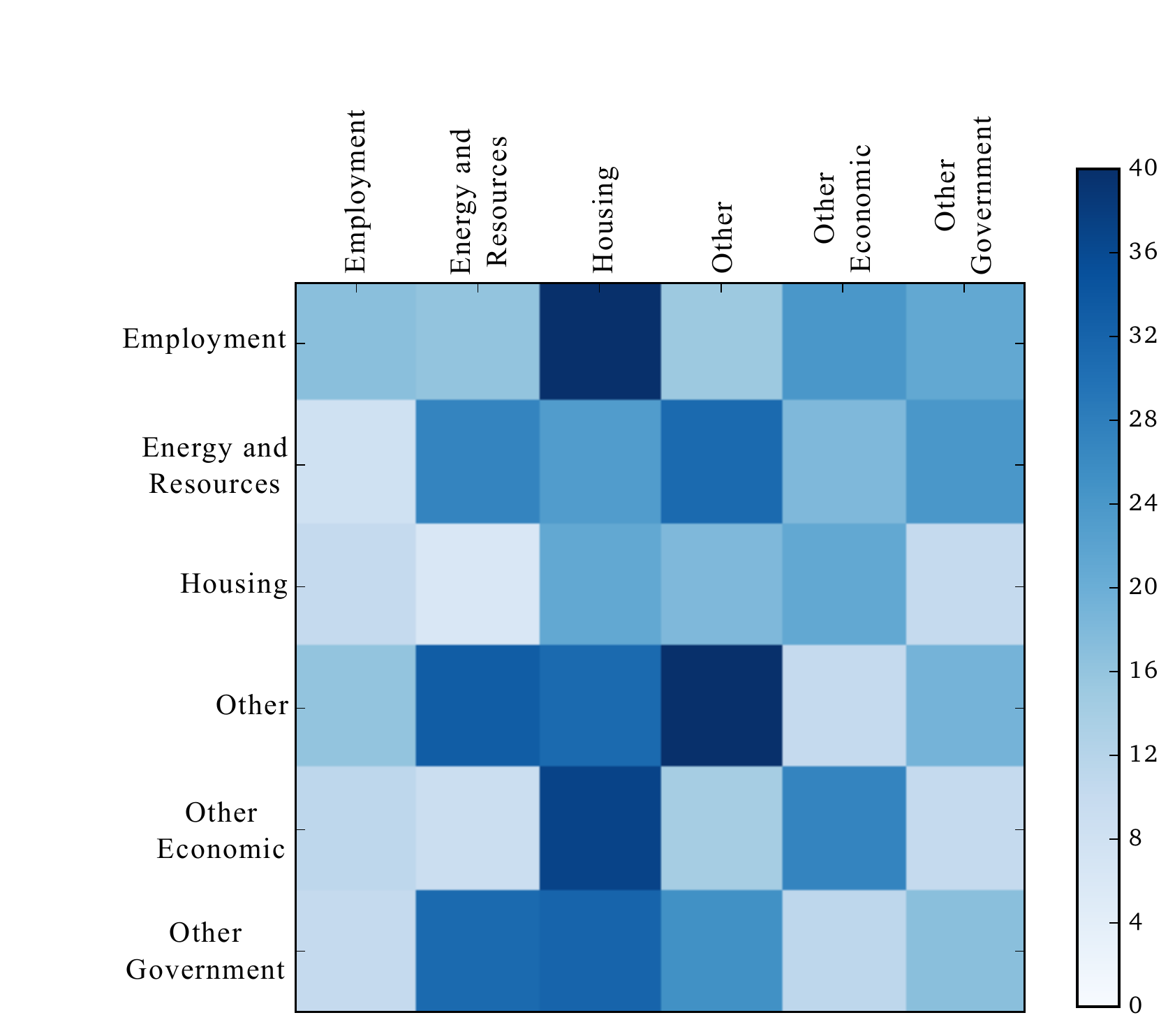}
    \caption{Influence of lagged targets on current targets}
    \label{fig:cu_heatmap:target}
\end{subfigure}\\
\begin{subfigure}{0.9\columnwidth}
  \centering
    \includegraphics[width=1\columnwidth]{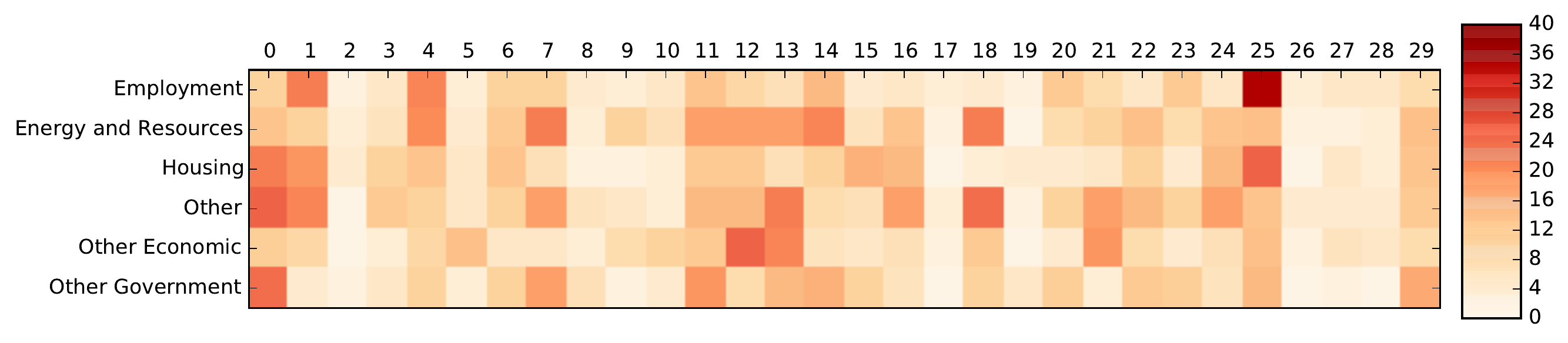}
    \vspace{-2em}
    \caption{Influence of lagged surrogates on current targets}
    \label{fig:cu_heatmap:surr}
\end{subfigure}
\vspace{-7pt}
\caption{(Brazilian Spring) Heatmap of changepoint influences of  targets on
targets (a); and surrogates on targets (b). Darker (lighter) shades indicate
higher (lesser) changepoint influence. (a) shows presence of strong
off-diagonal elements indicating strong cross-target changepoint information.
(b) shows a mixture of uninformative and informative surrogates.
\label{fig:cu_heatmap}}
\vspace{-20pt}
\end{figure}

\subsubsection{Changepoint influence analysis}
\label{ssub:causalaity}
The experiments presented in the previous section can
be further analyzed to ascertain the nature of 
progression of significant events that lead to a protest. Here
we present our analysis for \textit{Brazilian Spring}.
As a preliminary step, from Table~\ref{tb:cu_results} we can see that the
detected changepoints for Brazil reveal an interesting progression -
significant changes in Energy related unrests (06/02) propagated to
Housing/Other Govt. unrests (06/16) and culminated in mass Employment related
unrests (08/18). Interestingly, we can analyze the fitted 
parameters of the weight vector $A^i_{0/1}$ of the rate updates 
(see~\ref{eq:cpd_dist_dyn:data}) to quantize the changepoint influence of a 
source (target/surrogate) at time $T-1$ to time $T$. For each target $S_i$, we can compute the average
value of the weight vector component of each target/surrogate separately. 
Let $h_{0}$ and $h_{1}$ denote these averages for one such source. Effectively,
$h_{0}$ then measures the effect of the source at time
$t-1$ on $S_i$ at $t$ before
change while $h_{1}$ captures the same post change. Their percentage
relative change 
can then be used as a measure of the changepoint
influence of a particular target/surrogate source on $S_i$.
We plot a heatmap of these percentages in Figure~\ref{fig:cu_heatmap} for both
targets and surrogates, separately.
From Figure~\ref{fig:cu_heatmap:target}, we can see that 
`Other Economic' and 'Employment' related protests 
had strong influences from `Housing' related protests.
Furthermore, from Figure~\ref{fig:cu_heatmap:surr} 
we can see `Housing' and `Employment' related protests were influenced
by similar Twitter chatter clusters (cluster-01 and cluster-26) - indicating
that the interaction between these protest subtypes can be inferred from social
domain. Conversely, `Housing' and `Other Economic' related protests are only
weakly correlated through Twitter chatters - thus exhibiting the robustness
of {\ProposedModel} which can still detect interactions between targets
when surrogates fail to explain the same.
In general, for a particular target we can see  
linked~\textit{pre-cursors in other targets} (strong 
off-diagonal elements in Figure~\ref{fig:cu_heatmap:target}) and
highly specific
\textit{informative surrogates} (few strong cells for a row in 
Figure~\ref{fig:cu_heatmap:surr}).

\section{Conclusion}

We have presented \ProposedModel, a framework for online change detection in
multiple data sources which can augment additional surrogate information
sources in a hierarchical manner. Key properties of
our framework are the following a) it is computationally inexpensive requiring
a search over local change points (for each data layer) making it applicable
for a large number of data/surrogate sources, b) the change detection
algorithms are readily tunable to account for different false alarm
requirements at different data layers, and c) it provides a systematic approach
to integrate surrogate information for the same. As shown through a variety of
experiments on both synthetic and real world data sets, the proposed approach
uncovers interesting relationships and 
significantly outperforms state-of-the-art methods which do not account for
surrogate information both in terms of event detection reliability (probability
of false alarm) as well as the delay in detection.

\vspace{1em}
\begin{normalsize}
\noindent
{\textbf{Supporting Information} A demo of {\ProposedModel} and 
  the datasets used in this
  paper can be found in \url{https://prithwi.github.io/hqcd_supplementary}.
  Attached appendix provides additional details on {\SMC}.
}
\end{normalsize}

\begin{normalsize}
\noindent
{\textbf{Acknowledgements}
Supported by the Intelligence Advanced Research Projects Activity       
(IARPA) via Department of Interior National Business Center (DoI/NBC)   
contract number D12PC000337, the US Government is authorized to         
reproduce and distribute reprints of this work for Governmental         
purposes notwithstanding any copyright annotation thereon.              
Disclaimer: The views and conclusions contained herein are those of     
the authors and should not be interpreted as necessarily representing   
the official policies or endorsements, either expressed or implied, of  
IARPA, DoI/NBC, or the US Government.                                   
}
\end{normalsize}

\begin{appendix}
\section{Sequential Bayesian Inference}
  \label{ap:sbi}
Consider a stochastic process where an observed temporal data sequence
$\vec{y} = \vbrace{y_1, y_2, \ldots , y_t}$
depends on unobserved latent states
$\vec{x} = \vbrace{x_1, x_2, \ldots, x_t}$ 
such that the following formulation holds:
\begin{equation}
  \label{eq:smc2:model}
  \begin{array}{rcl}
    P(y_{t}|y_{1:t-1}, x_{1: t}, \theta) & = & f_{\theta}(y_{t} | x_{t}) \\
    P(x_{t}|x_{1:t-1}, \theta) & = & g_{\theta}(x_{t}|x_{t-1}) \\
    P(x_1|\theta) & = & \mu_{\theta}(x_1) \\
    \Pi_0(\theta) & = & P(\theta)\\    
  \end{array}
\end{equation}
i.e. $y_t$ depends only on the current estimate of the state $x_t$.  On the
other hand, $x_t$ depends only on $x_{t-1}$, thus exhibiting a first-order
Markov property.  $\theta$ denotes the set of parameter for the described
process which are constant over time. For some $\theta$, $f_\theta$, $g_\theta$
describe the observation probability and the state transition probability,
respectively.  $P(\theta)$ is the prior distribution for the static parameter
$\theta$ while $\mu_{\theta}$ is the same for $x$ given a particular $\theta$.
Typically, at any time point $t-1$ the observation values are known but the
latent states and the parameter $\theta$ are unknown. The problem of interest
is then to \textit{estimate the posterior probability} 
\begin{align}
P_\theta(\vbrace{x_1, x_2, \ldots, x_{t-1}} | \vbrace{y_1, y_2, \ldots, y_{t-1}})\nonumber
\end{align}
This problem has been studied extensively in the context of Sequential Bayesian
Inference~\citep{casella2002statistical}.
Kalman filters~\citep{kalman1960new}, a class of such
algorithms, are very popular
when $f_\theta$ and $g_\theta$ describe linear Gaussian transitions. There have been
efforts~\citep{simon2010survey, anderson2001ensemble}
at relaxing these restrictions  using 
methods such as Taylor series expansion and ensemble averages. However, for
arbitrary forms of $f_\theta \mbox{ and } g_\theta$, Sequential Monte Carlo and more specifically
Particle Filters are more popular. 
Particle Filters~\citep{del1996non} estimate the posteriors using a large number of Monte Carlo samples from the
observation and state transition models. At any time $t$, these algorithms only
need to draw new samples for time $t$ using data from only $t-1$. Thus these
methods are ideally suited for online learning. 
Standard Particle Filters are known to suffer from premature convergence 
(particle degeneracy)~\citep{doucet2009tutorial} or unsuitable 
for unknown static variables~\citep{pitt1999filtering,doucet2009tutorial}
Recently, Chopin \etal~\citep{Chopin:smc2} proposed a hybrid Particle filter
which interleaves Iterated batch resampling with particle filter updates to
handle both static and state parameters. Given an observed sequence $y_{1:t}$,
\SMC~can be used to find the best posterior fit of the static and state
parameters as given below:
\[P\left(\phi, \{x_{1:t}\}^{\phi}\ |\ y_{1:t}\right)\]

\subsection{\SMC~algorithm traces}
We present the traces of the \SMC~algorithm below.
For a more detailed treatment of the same
(including theoretical proofs of convergence) we ask
the readers to refer to~\citep{Chopin:smc2}.

\SMC~typically starts with two parameters: (a) $N_\theta$ - the number of
static parameters sampled from the prior of $\theta$ and (b) $N_x$ - the number
of particles of initialized for each $\theta$.

Then the Algorithm can be given as follows:

\begin{enumerate}
  \item Sample $N_\theta$ number of $\theta^{m} \sim P(\theta)$
  \item $\forall \theta^{m}$ run the following particle filter 
    \begin{enumerate}
      \item \texttt{Initialization: }$t=1$ 
        \begin{enumerate}
          \item $x_1^{1:N_x, m} \sim \mu_{\theta^{m}}$
          \item $w_{1, \theta}(x_1^n, m) = \frac{\mu_{1,\theta^m}(x_1^{n,m})g_\theta(y_1|x_1^{n,m}}{q_{1, \theta}(x_1^{n,m)}} $
          \item $W_{1, \theta}^{n,m} = \frac{w_{1, \theta}(x_1{n, m})}{\sum_i w_{1, \theta}(x_1{i, m})}$
          \item $P(y | \theta^{m}) = \frac{1}{N_x} \sum \limits_{n=1}^{N_x} w_{1, \theta}(x_1^{n,m})$
        \end{enumerate}
      \item $t \geq 1$
        \begin{enumerate}
          \item \texttt{Auxiliary variable: \\}$a_{t-1}^{n,m} \sim Multinomial\left( W_{t-1, \theta}^{1:N_x, m}\right)$
          \item \texttt{State Proposal: \\}$ x_t{n, m} \sim q_{t, \theta}\left(.|x_{t-1}^{a_{t-1}^{n, m}} \right)$
          \item \texttt{Weight Update: \\}$W_{t,\theta}\left( x_{t-1}^{a_{t-1}^{n,m}}\right) \sim \frac{w_{t,\theta}\left(x_{t-1}^{a_{t-1}^{n,m}}x_{t}^{n,m}\right)}
            {\sum x_{t-1}^{a_{t-1}^{n,m}}x_{t}^{n,m}}$
          \item \texttt{Observation probability: \\}
            $P(y_t | y_{1:t-1}, \theta^m) = \frac{\sum \limits_{n=1}^{N_x} w_{t,\theta}\left(x_{t-1}^{a_{t-1}^{n,m}}x_{t}^{n,m}\right)}{N_x} $
        \end{enumerate}
    \end{enumerate}
  \item {\bf Update Importance weights: \\}
   $\forall \theta^m \> w^{m} \leftarrow w^m P(y_t | y_{1: t-1, \theta^m}$
 \item {\bf Under degeneracy criterion:\\} \texttt{Move particles using Kernel}
   \begin{equation*}
     \begin{array}{ll}
       \left(\tilde{\theta^m}, \tilde{x_{1:t}^{1:N_x, m}}, \tilde{a_{1:t-1}^{1:N_x, m}} \right) \iid & \qquad\\
       \qquad \qquad
       \frac{\sum_ w^{m}\mathcal{K}_t\left({\theta^m}, {x_{1:t}^{1:N_x, m}}, {a_{1:t-1}^{1:N_x, m}} \right)}{\sum_m w^m}  & \qquad
     \end{array}
   \end{equation*}
 \item {\bf Weight Exchange: \\}
   $\left({\theta^m}, {x_{1:t}^{1:N_x, m}}, {a_{1:t-1}^{1:N_x, m}} \right) \leftarrow
     \left(\tilde{\theta^m}, \tilde{x_{1:t}^{1:N_x, m}}, \tilde{a_{1:t-1}^{1:N_x, m}} \right)$ 
\end{enumerate}

Here, $\mathcal{K}$ is a Markov kernel Targeting the posterior distribution. It can be shown that such Markov moves
don't change the target distribution and can alleviate the problem of particle degeneracy.

\subsection{\SMC~priors}
We used conjugate distributions to model the priors. 
For, $P(\theta)$ we used a mixture of Latin hypercube
sampling ($LHS$) and conjugate priors as follows:
\vspace{-8pt}
\begin{equation}
  \label{eq:smc:prior_theta}
  \begin{array}{rcl}
    \sigma_S, \vec{\rho_S}, \vec{\mu^1}, \vec{\mu^2} & \sim & LHS \\
    \Sigma_A & \sim & \text{InverseWishart},\ 
  \end{array}
\end{equation}

Similar to $P(\theta)$, we model the initial distribution
$P(x_0|\theta)$ via LHS sampling for the base values and by using the
model equations as presented in Section~{3.1}.
as follows:
\begin{equation}
  \label{eq:smc:prior_x}
  \begin{array}{rcl}
    \vec{c^K} & \sim & \text{Normal}\\
    \vec{\phi^k}, \vec{\rho^s}, & \sim & \text{Gamma}\\
  \end{array}
\end{equation}
The parameters of the distributions of $P(\theta)$ and $P(x_0|\theta)$ are 
called hyperparameters in the general domain of Bayesian Inference and
following standard practices are found via cross-validation.
\end{appendix}

\end{document}